\pdfoutput=1

\RequirePackage{snapshot}
\documentclass[11pt]{article}

\usepackage{ACL2023}

\usepackage{times}
\usepackage{latexsym}

\usepackage[T1]{fontenc}

\usepackage[utf8]{inputenc}

\usepackage{microtype}

\usepackage{inconsolata}

\usepackage{hyperref}

\usepackage{tipa}
\newcommand{\ucambridge}{\normalfont \text{\textipa{D}}}

\newcommand{\ethz}{\text{\normalfont \textipa{Q}}}
\newcommand{\mitboston}{\normalfont \text{\textipa{@}}}

\newcommand{\jhu}{\normalfont \text{\textipa{6}}}

\usepackage{amsmath}
\usepackage{amssymb}
\usepackage{amsfonts}
\usepackage{amsthm}
\usepackage{mathtools}
\usepackage{xspace}
\usepackage{subcaption}
\usepackage{bm}
\usepackage{enumerate}
\usepackage[shortlabels]{enumitem}
\setlist{nosep}
\usepackage{tikz}
\usepackage{float}
\usepackage{mathrsfs}

\usetikzlibrary{shapes,arrows,arrows.meta,positioning,automata}

\tikzset{
-{Stealth[length=2mm, width=2mm]}, 
node distance=1cm,
every state/.style={thick, fill=gray!10},
initial text=$ $,
}

\allowdisplaybreaks

\usepackage[capitalize,noabbrev]{cleveref}
\crefname{section}{\S}{\S\S}
\crefname{table}{Tab.}{Tables}
\crefname{figure}{Fig.}{Figures}
\crefname{algorithm}{Alg.}{Algs.}
\crefname{equation}{Eq.}{Eqs.}
\crefname{definition}{Def.}{Definitions}
\crefname{appendix}{App.}{Apps.}
\crefname{theorem}{Thm.}{Theorems}
\crefname{myexample}{Ex.}{Examples}
\crefname{proposition}{Prop.}{Propositions}
\crefname{corollary}{Cor.}{Corollaries}
\crefname{observation}{Observation}{Observations}
\crefname{assumption}{Assumption}{Assumptions}
\crefname{hypothesis}{Hyp.}{Hypotheses}
\crefname{footnote}{footnote}{footnotes}   
\crefformat{section}{\S#2#1#3}
\newcommand{\defeq}[0]{\mathrel{\stackrel{\textnormal{\tiny def}}{=}}}

\usepackage{thmtools} 
\usepackage{thm-restate}
\usepackage{etoolbox}
\theoremstyle{plain}
\declaretheorem[name=Theorem,numberwithin=section]{theorem}
\newtheorem{proposition}[theorem]{Proposition}

\newtheorem{lemma}[theorem]{Lemma}

\newtheorem{corollary}[theorem]{Corollary}
\newtheorem{definition}[theorem]{Definition}

\theoremstyle{remark}

\theoremstyle{definition}
\newtheorem{myexample}[theorem]{Example}
\AtBeginEnvironment{myexample}{\pushQED{\qed}}
\AtEndEnvironment{myexample}{\popQED\endmyexample}

\newcommand{\seq}[1]{\texttt{#1}}

\newcommand{\vsource}{\boldsymbol{s}}
\def\vh{\boldsymbol{h}}

\def\vu{\boldsymbol{u}}

\def\vv{\boldsymbol{v}}
\newcommand{\vtarget}{\boldsymbol{t}}

\newcommand{\alphabet}{\Sigma}
\newcommand{\kleene}[1]{#1^*}

\newcommand{\overC}{\overline{C}}
\newcommand{\overcalC}{\overline{\mathcal{C}}}

\newcommand{\vertiii}[1]{{\left\vert\kern-0.25ex\left\vert\kern-0.25ex\left\vert #1 
    \right\vert\kern-0.25ex\right\vert\kern-0.25ex\right\vert}}

\newcommand{\ueos}{\bm{u}_{\eos}}

\usepackage[disable]{todonotes}
\makeatletter
\newcommand*\iftodonotes{\if@todonotes@disabled\expandafter\@secondoftwo\else\expandafter\@firstoftwo\fi}
\makeatother
\newcommand{\noindentaftertodo}{\iftodonotes{\noindent}{}}

\newcommand{\Ryan}[2][]{\noindentaftertodo}
\newcommand{\Tiago}[2][]{\noindentaftertodo}
\newcommand{\Clara}[2][]{\noindentaftertodo}
\newcommand{\Lucas}[2][]{\noindentaftertodo}
\newcommand{\Leo}[2][]{\noindentaftertodo}
\newcommand{\Jason}[2][]{\noindentaftertodo}

\DeclareSymbolFont{bbold}{U}{bbold}{m}{n}
\DeclareSymbolFontAlphabet{\mathbbold}{bbold}

\newcommand{\justification}[1]{\text{{\color{gray}(#1)}}}

\renewcommand{\P}{\mathbb{P}}

\newcommand{\calC}{\mathcal{C}}
\newcommand{\powerset}[1]{\mathcal{P}(#1)}

\newcommand{\vx}{\bm{x}}

\renewcommand{\c}{\mathsf{c}}
\newcommand{\calF}{\mathcal{F}}
\newcommand{\calA}{\mathcal{A}}

\newcommand{\calG}{\mathcal{G}}

\newcommand{\relu}{\text{ReLU}}

\newcommand{\R}{\mathbb{R}}
\newcommand{\N}{\mathbb{N}}
\newcommand{\Z}{\mathbb{Z}}
\newcommand{\xx}{\boldsymbol{x}}
\newcommand{\yy}{\boldsymbol{y}}
\newcommand{\defn}[1]{\textbf{#1}}

\newcommand{\alphabeteos}{\overline{\alphabet}}
\newcommand{\eos}{\textsc{eos}\xspace}
\newcommand{\bos}{\textsc{bos}\xspace}

\newcommand{\pString}{p}
\newcommand{\pASM}{\bar{p}}
\newcommand{\pknownLM}{p_0}

\definecolor{darkblue}{rgb}{0.0,0.0,0.5}
\definecolor{purple}{rgb}{0.5,0.0,0.5}

\newcommand{\bmomega}{\bm{\omega}}
\newcommand{\infoft}{\text{ i.o.}}
\newcommand{\comp}{\mathsf{c}}
\def\mP{{\mathbf{P}}}

\def\mI{{\mathbf{I}}}
\newcommand{\autoregmodel}{autoregressive sequence model\xspace}

\newcommand{\autoregmodelAcronym}{ASM\xspace}
\newcommand{\softplus}{\mathrm{softplus}}
\newcommand{\ptildeEOS}{\widetilde{p}_\eos}

\newcommand{\boldX}{X}

\title{A Measure-Theoretic Characterization of Tight Language Models}

\author{%
Li Du$^{\jhu}$%
~\;~\;~Lucas Torroba Hennigen$^{\mitboston}$%
~\;~\;~Tiago Pimentel$^{\ucambridge}$ \\
\textbf{Clara Meister}$^{\ethz}$
~\;~\;~\textbf{Jason Eisner}$^{\jhu}$~\;~\;~\textbf{Ryan Cotterell}$^{\ethz}$\\
    $^{\jhu}$Johns Hopkins University~\;~\;~\;~$^{\mitboston}$MIT \\
  $^{\ucambridge}$University of Cambridge   ~\;~\;~\;~
  $^{\ethz}$ETH Z{\"u}rich
   \\
\texttt{\href{mailto:leodu@cs.jhu.edu}{leodu@cs.jhu.edu}}%
  ~\;~ \texttt{\href{mailto:lucastor@mit.edu}{lucastor@mit.edu}}%
  ~\;~ \texttt{\href{mailto:tp472@cam.ac.uk}{tp472@cam.ac.uk}} \\  \texttt{\href{mailto:clara.meister@inf.ethz.ch}{clara.meister@inf.ethz.ch}}%
  ~\;~ \texttt{\href{mailto:jason@cs.jhu.edu}{jason@cs.jhu.edu}}%
  ~\;~ \texttt{\href{mailto:ryan.cotterell@inf.ethz.ch}{ryan.cotterell@inf.ethz.ch}}
 }

\begin{document}
\maketitle
\begin{abstract}
Language modeling, a central task in natural language processing, involves estimating a probability distribution over strings.
In most cases, the estimated distribution sums to 1 over all finite strings.
However, in some pathological cases, probability mass can ``leak'' onto the set of infinite sequences.
In order to characterize the notion of leakage more precisely, this paper offers a measure-theoretic treatment of language modeling.
We prove that many popular language model families are in fact tight, meaning that they will not leak in this sense. 
We also generalize characterizations of tightness proposed in previous works.\looseness=-1

\end{abstract}

\section{Introduction}

Language modeling is a core task in natural language processing.
As canonically defined, language modeling involves estimating a distribution over the set of strings over a given alphabet.
If the alphabet is the set of words in a language,\footnote{Or perhaps alphabetic symbols or subwords; see, e.g., \citet{durrett-nlp}.} then a language model can be thought of as a distribution over the language's sentences.
Since \citet{shannon1948mathematical}, language modeling has been used to estimate statistical properties of language and has become essential for computational linguistics research \cite{hale2001probabilistic,meister-etal-2021-revisiting}.
Further, it is also central to a wide range of natural language processing applications, whether as a source model in a noisy channel architecture \cite{weaver1949,jelinek1976}, as a way of learning better representations of language \citep{peters-etal-2018-deep}, or,
more recently, for prompted generation \citep{gpt3}, where the distribution defined by a language model is employed in tasks like question-answering \citep{petroni-etal-2019-language}, style transfer \citep{reif-etal-2022-recipe}, and even sequence tagging \cite{liu+al.emnlp22}.\looseness=-1

More formally, a language model is typically defined to be a distribution over the \emph{countably} infinite set $\kleene{\alphabet}$ of all (finite) strings \cite{booth1973}.\footnote{Recall that $\kleene{\alphabet} \defeq \bigcup_n \alphabet^n$ where for $n \geq 0$, $\alphabet^n$ is the set of strings of length $n$.  The $^*$ is the Kleene closure operator.}
However, it has been shown that some classes of autoregressive language models have 
parameter settings in which the generative process terminates with probability $< 1$.
\Citet{welleck-etal-2020-consistency} discuss this issue for recurrent neural network  language models.  
Models whose generative process may fail to terminate are called \defn{non-tight} \cite[who discussed non-tight PCFGs]{chi-1999-statistical}.
If an autoregressive language model is non-tight, 
it may generate infinite sequences and MCMC algorithms over such models will not mix to the correct distribution.

It is here that a subtlety arises: 
the set of infinite sequences is \emph{uncountably} infinite.
Properly treating a distribution over this sample space requires a modicum of measure theory.\footnote{Indeed, our treatment resolves an imprecision present in the literature. 
For instance, \citet{welleck-etal-2020-consistency} discusses the probability of infinite sequences despite using the canonical definition of a language model as a distribution over $\kleene{\alphabet}$.}
To clarify the situation, we review the measure-theoretic treatment of distributions over 
infinite sequences.
We then make use of a termination symbol $\eos$ to define a random variable whose value can be a string, i.e., an element of $\kleene{\alphabet}$, or an infinite sequence.
In a tight language model, this random variable has probability 1 of being a string and hence finite.

Beyond offering a measure-theoretic formalization, our paper also demonstrates how tightness relates to the Borel--Cantelli lemmata, simplifying a recent result by \citet{meister-tacl2022}.
To conclude our paper, we analyze several language modeling architectures and give conditions on their tightness.
We demonstrate that $n$-gram language models---and more generally, language models defined by stochastic finite-state automata---can be non-tight, and we give a simple necessary and sufficient condition for tightness in terms of the inverse of the automaton's transition matrix.
This builds on a known result due to \citet{lehmann1977algebraic}.
We also discuss when neural language models become non-tight.
We prove that Transformer-based language models \cite{vaswani-2017-attention,gpt3} are always tight and that recurrent neural language models are always tight when they employ a bounded activation function.
However, we also exhibit a recurrent neural network (RNN) language model with a ReLU activation~\citep{relu} that is non-tight in a simpler construction than the one offered by \citet{chen-etal-2018-recurrent}.
As a byproduct, we also generalize and strengthen the results given by \citet{welleck-etal-2020-consistency}, who give a sufficient condition for tightness of recurrent neural language models in terms of the norm of the hidden state.\looseness=-1

\section{Motivating Examples}
\label{sec:examples}

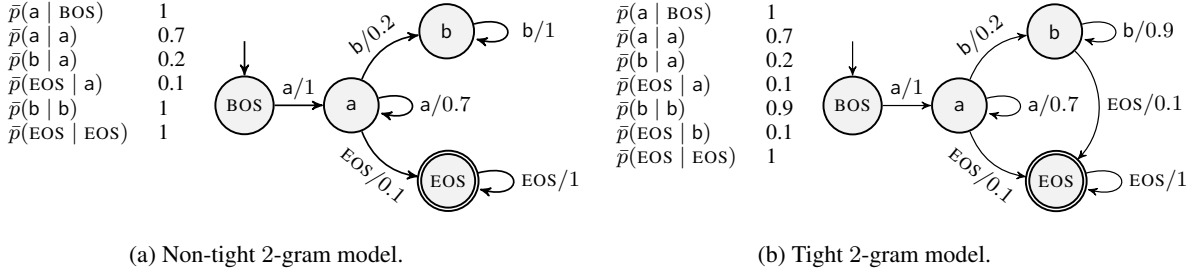
\begin{figure*}[ht]
    \begin{subfigure}{0.47\textwidth}
    \centering
    \scalebox{0.91}{
    \raisebox{-0.65\height}{
        \footnotesize
        \begin{tabular}{ll}
            $\pASM(\seq{a}\mid \textsc{bos})$ & 1 \\
            $\pASM(\seq{a}\mid \seq{a})$ & 0.7 \\
            $\pASM(\seq{b}\mid \seq{a})$ & 0.2 \\ 
            $\pASM(\eos\mid \seq{a})$ & 0.1 \\ 
            $\pASM(\seq{b}\mid \seq{b})$ & 1 \\ 
            $\pASM(\eos\mid\eos)$ & 1
        \end{tabular}}

      \raisebox{-0.8\height}{
        \footnotesize
        \begin{tikzpicture}[->,>=stealth',line width=0.7pt,node distance=1cm,minimum size=0.9cm,]
        \node[state] (a) {$\seq{a}$};
        \node[state,above right=0.5cm and 0.8cm of a] (b) {$\seq{b}$};
        \node[state,below right=0.5cm and 0.8cm of a,accepting] (eos) {{\footnotesize\eos}};
        \node[state,left=0.7cm of a] (bos) {{\footnotesize\bos}};
        \node[above=0.5cm of bos] (dummy) {};

        \path (dummy) edge node {} (bos);
        \path (bos) edge node [above,sloped, yshift=-6pt] {$\seq{a}/1$} (a);
        \path (a) edge [loop right] node [right] {$\seq{a}/0.7$} (a);
        \path (b) edge [loop right] node {$\seq{b}/1$} (b);
        \path (eos) edge [loop right] node {$\eos/1$} (eos);
        \path (a) edge [bend left] node [above, sloped, yshift=-6pt] {$\seq{b}/0.2$} (b);
        \path (a) edge [bend right] node [below, sloped, yshift=6pt] {$\eos/0.1$} (eos);
        \end{tikzpicture}
    }
    }
    \caption{Non-tight $2$-gram model.}
    \label{fig:non-tight-2-gram-example}
    \end{subfigure}
    \quad
    \begin{subfigure}{0.47\textwidth}
    \centering
    \scalebox{0.91}{
    \raisebox{-0.7\height}{
    \footnotesize
        \begin{tabular}{ll}
            $\pASM(\seq{a}\mid \textsc{bos})$ & 1 \\
            $\pASM(\seq{a}\mid \seq{a})$ & 0.7 \\
            $\pASM(\seq{b}\mid \seq{a})$ & 0.2 \\
            $\pASM(\eos\mid \seq{a})$ & 0.1 \\
            $\pASM(\seq{b}\mid \seq{b})$ & 0.9 \\
            $\pASM(\eos\mid \seq{b})$ & 0.1 \\
            $\pASM(\eos\mid\eos)$ & 1
        \end{tabular}}

      \raisebox{-0.8\height}{
        \footnotesize
        \begin{tikzpicture}[->,>=stealth',line width=0.5pt,node distance=1cm,minimum size=0.9cm,]
        \node[state] (a) {$\seq{a}$};
        \node[state,above right=0.5cm and 0.8cm of a] (b) {$\seq{b}$};
        \node[state,below right=0.5cm and 0.8cm of a,accepting] (eos) {{\footnotesize\eos}};
        \node[state,left=0.7cm of a] (bos) {{\footnotesize\bos}};
        \node[above=0.5cm of bos] (dummy) {};

        \path (dummy) edge node {} (bos);
        \path (bos) edge node [above,sloped, yshift=-6pt] {$\seq{a}/1$} (a);
        \path (a) edge [loop right] node [right] {$\seq{a}/0.7$} (a);
        \path (b) edge [loop right] node {$\seq{b}/0.9$} (b);
        \path (b) edge [bend left=45] node [right] {$\eos/0.1$} (eos);
        \path (eos) edge [loop right] node {$\eos/1$} (eos);
        \path (a) edge [bend left] node [above, sloped, yshift=-6pt] {$\seq{b}/0.2$} (b);
        \path (a) edge [bend right] node [below, sloped, yshift=6pt] {$\eos/0.1$} (eos);
        \end{tikzpicture}
    }
    }
     \caption{Tight $2$-gram model.}
     \label{fig:tight-2-gram-example}
    \end{subfigure}
    \caption{\label{fig:sfssm}Tight and non-tight bigram models, expressed as Mealy machines (see \cref{sec:sfssm}). Transitions with conditional probability of 0 are omitted. The termination probability at a state is represented by an \eos arc from that state.}
    \vspace{-0.65cm}
\end{figure*}

Let $\alphabet$ be an alphabet, i.e., a finite set of symbols,
and let $\eos \notin \alphabet$ be a distinguished end-of-sequence symbol.  Let 
$\alphabeteos \defeq \alphabet \cup \{\eos\}$. 
A \defn{string} of length $n \geq 0$ is a finite sequence $\vx = x_1 x_2 \ldots x_n$ where each $x_t \in \alphabet$. 
By convention, we say that $x_{n+1}=\eos$, although $x_{n+1}$ is not an element of the sequence $\vx$.
For any integer $1 \leq t \leq n+1$, we write $\vx_{<t}$ for the prefix $x_1 x_2\cdots x_{t-1}$.  

We now begin to distinguish between ``language models'' and ``sequence models.''
As is traditional in the NLP literature, we henceforth define a \defn{language model} to be a probability distribution over the countable set $\alphabet^*$ of all strings (see \cref{def:language-model}).
It is popular to specify such a distribution in terms of its conditional probabilities $\pASM(x_t \mid \vx_{<t})$.\looseness=-1
\begin{definition} 
\label{def:autoreg-model}
An \defn{\autoregmodel} (\defn{\autoregmodelAcronym}) is a conditional probability distribution $\pASM(x_t \mid \vx_{<t})$ where $x_t\in\alphabeteos$ and $\vx_{<t}\in\alphabeteos^*$.
\end{definition}
If $\pASM$ is an \autoregmodelAcronym, then we define a non-negative function $\pString$ over $\Sigma^*$ by
$\pString(\vx) \defeq \prod_{t=1}^{n+1}\pASM(x_t\mid\vx_{<t})=\pASM(\eos\mid\vx)\prod_{t=1}^{n}\pASM(x_t\mid \vx_{<t})$, where $n$ denotes the length of $\vx$. 

But is $\pString$ a language model?  Not always, since as we will see below, it is possible for $\pString(\Sigma^*) \defeq \sum_{\vx\in\alphabet^*}\pString(\vx) < 1$.  Of course this ``bad'' case never happens if the ASM's conditional probabilities are derived from a known LM, in which case $\pString$ simply recovers that LM.\footnote{That is, suppose the
\autoregmodelAcronym's conditional probabilities
match the conditional probabilities of some known language model $\pknownLM$: that is, $\pknownLM(X_t=x_t\mid X_1\ldots X_{t-1}=\vx_{<t}) = \pASM(x_t\mid\vx_{<t})$ whenever the former conditional probability is well-defined under the language model $\pknownLM$, i.e., whenever $x_t \in \alphabeteos$ and $\vx_{<t}\in\kleene{\alphabet}$ with $\pknownLM(X_1\ldots X_{t-1}=\vx_{<t}) > 0$.
Then by the chain rule of probability, $\pString(\vx)=\pknownLM(\vx)$ for each $\vx\in\alphabet^*$.  Thus $\pString=\pknownLM$, so $\pString$ is a language model.} 
In this case clearly $\pString(\alphabet^\ast)  = 1$.  
It follows that if $\pString(\alphabet^\ast) < 1$, then the \autoregmodelAcronym's conditional probabilities do \emph{not} match the conditional probabilities of any language model $\pknownLM$.

We now exhibit such a ``bad'' \autoregmodelAcronym.  Although the conditional probability distributions $\pASM(\cdot \mid \vx_{<t})$ each sum to 1 over $\alphabeteos$, they fail to combine into a model $\pString$ that sums to 1 over $\kleene{\alphabet}$ (i.e., a language model).
\begin{myexample}[non-tight bigram model]\label{ex:non-tight-2-gram}
Consider 
the bigram model defined in \cref{fig:non-tight-2-gram-example} over the alphabet $\alphabet = \{\seq{a}, \seq{b}\}$. Under this model, any finite string that contains the symbol \seq{b} will have probability 0, since $\pASM(\eos\mid\seq{b})=\pASM(\seq{a}\mid\seq{b})=0$. This implies $\pString(\alphabet^\ast) = \sum_{n=0}^\infty \pString(\seq{a}^n) = \sum_{n=0}^\infty (0.7)^n \cdot 0.1 = \frac{0.1}{1 - 0.7} = \frac{1}{3} < 1$.
\end{myexample}

\begin{myexample}[tight bigram model] \label{ex:tight-2-gram}
In contrast, in \cref{fig:tight-2-gram-example}, obtained from \cref{ex:non-tight-2-gram} by changing the arcs from the \seq{b} state, $\pString(\alphabet^*)=1$. See \cref{sec:tight-2-gram-proof} for details of this calculation.
\end{myexample}

\cref{ex:non-tight-2-gram} above confirms that the autoregressive formulation does not necessarily yield $\pString$ that is a valid distribution over $\alphabet^*$.

But if $\pString$ is not a language model, what is it?  It is intuitive to suspect that, in a model with $\pString(\alphabet^*)<1$, the remainder of the probability mass ``leaks'' to infinite sequences, i.e., the generative process may continue forever with probability $>0$.
We will make this intuition formal in \cref{sec:measure-lm}.
By analogy with \citet{chi-geman-1998-estimation} and \citet{cohen-johnson-2013-effect}, we refer to such models as \defn{non-tight}.\footnote{In \citet{chi-geman-1998-estimation} and \citet{cohen-johnson-2013-effect}, a PCFG is non-tight if its generative process may not terminate, and consequently the total probability of all finite trees is less than 1. 
}  

The non-tightness of \cref{ex:non-tight-2-gram} is related to the fact that the probability of $\eos$ is $0$ at some states, in contrast to \cref{ex:tight-2-gram}. 
However, requiring $\pASM(\eos \mid \vx_{<t}) > 0$ for all prefixes $\vx_{<t}$ is neither necessary nor sufficient to ensure tightness. 
It is not \emph{necessary} because one can, for example, construct an \autoregmodelAcronym 
in which $\pASM(\eos \mid \vx_{<t}) = 0.1$ when $t$ is even but $= 0$ otherwise.  Such a model generates only odd-length strings but is tight.
It is not \emph{sufficient} because of the following example, in which $\pASM(\eos \mid \vx_{<t})$ is always positive but decays so rapidly toward $0$ that the generative process might continue forever.
\begin{myexample}[non-tight RNN] \label{ex:non-tight-rnn}
Consider an RNN over a small alphabet $\alphabeteos = \{\seq{a}, \eos\}$ with the following hidden state recurrence:
\begin{align}
    h_0 &= 0, &
    h_t &= \relu(h_{t-1}+1).
\end{align}
In this case, the hidden state admits a closed-form expression $h_t=t \in \mathbb{R}$. 
Setting the (1-dimensional) embedding of the alphabet to be $v_\seq{a}=1$ and $v_\eos=0$, we arrive at
\begin{align}
    \pASM(\eos& \mid \vx_{<t}) 
    = \mathrm{softmax}(v_\seq{a}h_t, v_\eos h_t)_\eos  \nonumber \\
    &= \textstyle\frac{e^{0\cdot t}}{e^{1\cdot t} + e^{0 \cdot t}}
    = \frac{1}{e^t + 1} > 0.
\end{align}
The \eos probability is always strictly positive, but \Cref{prop:lm-tight-main} shows that this sequence model is non-tight. Numerically, $\pString(\kleene{\alphabet}) \approx 0.702 < 1$. 
\end{myexample}
On the other hand, an \autoregmodelAcronym may be tight after all if the probability of $\eos$ decays more slowly---even when it still approaches $0$.   
\begin{myexample}[tight RNN] \label{ex:tight-rnn}
Consider again an RNN over the alphabet $\alphabeteos=\{\seq{a},\eos\}$ with the following recurrence using $\softplus$ activation:\footnote{We use $\softplus$ instead of \relu{} to simplify arithmetics.}
\begin{align}
    h_1&=0, &
    h_t&=\log(\exp(h_{t-1}) + 1).
\end{align}
Starting from $h_1=0=\log 1$, a simple induction argument shows that
\begin{align}
    h_t = \log (\exp \log (t-1)+1)=\log t.
\end{align}
Again, setting $v_\seq{a}=1$ and $v_\eos=0$, we arrive at
\begin{align}
    \pASM(\eos \mid \vx_{<t}) &= \mathrm{softmax}(v_\seq{a}h_t, v_\eos h_t)_\eos  \\
    &= \textstyle\frac{e^{0\cdot \log t}}{e^{1\cdot \log t}+e^{0\cdot\log t}} = \frac{1}{t+1} > 0.  \nonumber
\end{align}
This decays slowly to 0: $\lim_{t \to \infty} \pASM(\eos \mid \vx_{<t}) = 0$, but since $\sum_{t=1}^\infty  \pASM(\eos \mid \vx_{<t}) = \infty$,  \cref{prop:div-implies-tight} below implies that this \autoregmodelAcronym is tight.
\end{myexample}

Finally, we illustrate the peril of not treating distributions over uncountable sets carefully.
\begin{myexample}[infinite coin toss]\label{ex:inf-coin-toss}
Consider the infinite independent fair coin toss model, where we aim to place a distribution over $\{\seq{H},\seq{T}\}^\infty$, the uncountable set of infinite sequences of $\{\seq{H},\seq{T}\}$. Intuitively, such a distribution corresponds to an ASM in which for all $\vx_{<t}$, $\pASM(\seq{H} \mid \vx_{<t}) = \pASM(\seq{T} \mid \vx_{<t}) = \frac{1}{2}$ and $\pASM(\eos \mid \vx_{<t})=0$.
Clearly, each individual infinite sequence over $\{\seq{H},\seq{T}\}$ should be assigned probability $(\frac{1}{2})^\infty = 0$. Without a formal foundation, one may arrive at the following paradox:
\begin{align}
    1&=\pString\left(\{\seq{H},\seq{T}\}^\infty\right) 
    =\textstyle\pString\left(\bigcup_{\bmomega\in\{\seq{H},\seq{T}\}^\infty} \{\bmomega\}\right) \\
    &=\sum_{\bmomega\in\{\seq{H},\seq{T}\}^\infty}\pString(\{\bmomega\}) = \sum_{\bmomega\in\{\seq{H},\seq{T}\}^\infty} 0 \stackrel{?}{=} 0.%
    \nonumber \tag*{\qedhere}%
\end{align}
\end{myexample}
Together, these examples suggest that one must take care to characterize tightness.
And, to the authors' surprise, it does not appear as if such a careful characterization yet exists in the NLP literature.\looseness=-1

\section{The Language Model Measure}
\label{sec:measure-lm}

In this section, we rigorously characterize the kind of distribution induced by an \autoregmodelAcronym. As mentioned earlier, an \autoregmodelAcronym can lose probability mass to the set of infinite sequences, $\alphabet^\infty$. However, $\alphabet^\infty$, unlike $\alphabet^*$, is uncountable, and it is due to this fact that we need to work explicitly with the measure-theoretic formulation of probability.

\subsection{Measure-Theoretic Background}
\label{sec:background-measure-theory}
The goal of measure-theoretic probability is to assign probability to subsets
of an \defn{outcome space} $\Omega$. For instance, in \cref{ex:inf-coin-toss}, $\Omega=\{\seq{H},\seq{T}\}^\infty$.
However, in the course of the study of measure theory, it has become clear that for many common $\Omega$, it is impossible to assign probabilities in a way that satisfies a set of reasonable desiderata.\footnote{Measure theory texts commonly discuss such desiderata and the dilemmas that comes with them. See, e.g., Chapter 7 in \citet{tao2016analysis}, Chapter 3 in \citet{royden1988real} or Chapter 3 in \citet{billingsley1986}. We also give an example in \cref{thm:ac-ch-impossible}.} 
Consequently, the standard approach to probability theory resorts
to only assigning probability to certain ``nice'' subsets of $\Omega$, which are referred to as \defn{events} or \defn{measurable subsets}, as in the theory of integration or functional analysis.
The set of measurable subsets is commonly denoted as $\calF$ (\cref{def:sigma-algebra}), and a probability measure $\P:\mathcal{F}\to[0,1]$ is the function that assigns a probability to each measurable subset. 
As it turns out, the following simple and reasonable requirements imposed on $\calF$ and $\P$ are enough to rigorously discuss probability \cite{kolmogorov1933}.

\begin{definition}
\label{def:sigma-algebra}
Let $\powerset\Omega$ be the powerset of $\Omega$.
Then $\calF \subseteq \powerset\Omega$ is called a \defn{$\boldsymbol{\sigma}$-algebra} (or $\sigma$-field)
over $\Omega$ if the following conditions hold:
\begin{enumerate}[1)]
    \item $\Omega\in\calF$,
    \item if $E\in\calF$, then its complement $E^\c \in \calF$,
    \item if $E_1,E_2,\dots$ is a finite or infinite sequence of sets in $\calF$, then $\bigcup_n E_n \in\calF$.
\end{enumerate}
If $\calF$ is a $\sigma$-algebra over $\Omega$, we call the tuple $(\Omega,\calF)$ a \defn{measurable space}.
\end{definition}
A measurable space guarantees that operations on countably many sets are always valid, and hence permits the following definition.
\begin{definition}
\label{def:probability-measure}
A \defn{probability measure} $\P$ over a measurable space $(\Omega,\calF)$ is a function $\P:\mathcal{F}\to[0,1]$ such that
\begin{enumerate}[1)]
    \item $\P(\Omega)=1$,
    \item if $E_1,E_2,\dots$ is a finite or infinite sequence of disjoint sets in $\calF$, then $\P(\bigcup_n E_n)=\sum_n \P(E_n)$.
\end{enumerate}
In this case we call $(\Omega,\calF, \P)$ a \defn{probability space}.  Note that it assigns measure only to the sets in $\mathcal{F}$; other sets are said to be \defn{non-measurable}.
\end{definition}

\subsection{Sequence Models}
\label{sec:background-language-model}

As we saw in \cref{sec:examples}, sampling successive symbols from a non-tight \autoregmodelAcronym has probability $> 0$ of continuing forever.
Hence, we hope to regard the \autoregmodelAcronym as defining a probability space over $\Omega = \kleene{\alphabet} \cup \alphabet^\infty$, where $\alphabet^\infty$ denotes the set of infinite strings\footnote{We will use the phrase ``infinite string'' in this paper when it is natural to do so, e.g., in the context of $\kleene{\alphabet} \cup \alphabet^\infty$. However, this is nonstandard terminology:  in computer science, \emph{string} generally refers to a finite object.} over $\alphabet$.
Note that this set $\Omega$ is uncountable whenever $|\alphabet| \geq 2$.
We will first need to turn it into a measurable space by defining an appropriate $\sigma$-algebra.  

This type of distribution is more general than a language model, which takes $\Omega$ to be the set $\alphabet^*$ of finite strings. 
To distinguish the two, we refer to a distribution over $\kleene{\alphabet} \cup \alphabet^\infty$ as a sequence model.  

\begin{definition} 
\label{def:sequence-model}
A \defn{sequence model} is a probability measure $P$ over the set $\kleene{\alphabet} \cup \alphabet^\infty$.
\end{definition}
Intuitively (we will make this precise later), the event $\alphabet^\infty\subset\kleene{\alphabet} \cup \alphabet^\infty$ in \cref{def:sequence-model} represents \defn{non-termination} of the generating process, i.e., it attempts to generate an infinitely long sequence. If this never happens, we have a language model.

\begin{definition}\label{def:language-model}
A \defn{language model} is a probability measure $P$ over just $\alphabet^*$.
Equivalently, it is a sequence model $P$ such that $P(\alphabet^\infty)=0$.
\end{definition}
Our goal in the rest of this section is to rigorously construct a sequence model $P$ 
that encodes the conditional probabilities of a given \autoregmodelAcronym. Since the \autoregmodelAcronym specifies conditional distributions over the augmented alphabet $\alphabeteos$, we first use it to construct a probability measure $\P$ over a measurable space $(\alphabeteos^\infty, \sigma(\overcalC))$.  We then derive our sequence model $P$ from $\P$ as the probability measure of a random variable $X$ in a measurable space $(\kleene{\alphabet} \cup \alphabet^\infty, \sigma(\calC))$.  The $\sigma$-algebras $\sigma(\overcalC)$ and $\sigma(\calC)$ will be built below.

\subsection{Pre-Measure} \label{sec:pre-measure}

As mentioned in \cref{sec:background-measure-theory}, it is often impossible to measure the probability of every single subset of $\Omega$. 
For example, in the infinite coin toss model in \cref{ex:inf-coin-toss}, we might begin by reasonably assigning probability 0 to each individual sequence $\bmomega\in\{\seq{H}, \seq{T}\}^\infty$.
Unfortunately, it is then impossible to assign probability to \emph{every} subset of $\{\seq{H}, \seq{T}\}^\infty$; we must restrict our measurable space to a strict subset of $\powerset{\Omega}$,
where $\powerset{}$ is the powerset operator.%

\begin{restatable}{theorem}{ACCHImpossible} \label{thm:ac-ch-impossible}
Assuming the Axiom of Choice and the Continuum Hypothesis, there exists no probability measure $\P$ over $(\{\seq{H}, \seq{T}\}^\infty, \powerset{\{\seq{H}, \seq{T}\}^\infty)}$ such that $\P(\{\bmomega\})=0$ for each $\bmomega\in\{\seq{H}, \seq{T}\}^\infty$.
\end{restatable}
\begin{proof} This is a direct consequence of \citet{ulam1930}. See \cref{sec:ac-ch-impossible-proof} for a discussion.\looseness=-1
\end{proof}

We will address this with well-known methods. A versatile theorem of Carath\'{e}odory provides a natural way to construct a probability space for sequences, in which prefix probabilities are well-defined. We first review two needed definitions.%
\begin{definition}
$\calA \subseteq \powerset{\Omega}$ is called an \defn{algebra} (or field) over $\Omega$ if
\begin{enumerate}[1)]
    \item $\Omega\in\calA$,
    \item if $E\in\calA$, then $E^\c \in \calA$,
    \item if $E_1,E_2 \in \calA$, then $E_1 \cup E_2 \in \calA$.
\end{enumerate}
\end{definition}

\begin{definition}
Let $\calA$ be an algebra over some set $\Omega$. A \defn{probability pre-measure} over $(\Omega,\calA)$ is a function $\P_0:\calA \to [0,1]$ such that
\begin{enumerate}[1)]
    \item $\P_0(\Omega)=1$,
    \item if $E_1,E_2,\dots$ is a (countable) sequence of disjoint sets in $\calA$ \emph{whose (countable) union is also in} $\calA$, then $\P_0(\cup_{n=1}^\infty E_n)=\sum_{n=1}^{\infty} \P_0(E_n)$.
\end{enumerate}
\end{definition}

Note that the only difference between a $\sigma$-algebra (\cref{def:sigma-algebra}) and an algebra is that condition 3 is weakened from countable to finite, and the only difference between a probability measure (\cref{def:probability-measure}) and a pre-measure is that the latter is defined with respect to an algebra instead of a $\sigma$-algebra.

The idea behind Carath\'eodory's Extension Theorem is that there is often a simple construction of an algebra $\calA$ over $\Omega$ such that there is a natural way to define a probability pre-measure. One can then extend this probability pre-measure to a probability measure that is both minimal and unique in a precise sense. For example, the standard Lebesgue measure over the the real line can be constructed in this way.
For our case of infinite sequences, we will first construct an algebra over $\Omega=\alphabeteos^\infty$ for some alphabet $\alphabet$. Then, assuming we are given an \autoregmodelAcronym $\pASM$ over $\alphabeteos$, we can associate the algebra with a pre-measure that is consistent with $\pASM$. We will make use of the following definition to construct the algebra:

\begin{definition}\label{def:cylinder-set}
Given any set $H\subseteq\alphabeteos^k$, define its \defn{cylinder set} (of \defn{rank} $k$) to be
\begin{equation}
\overC(H)\defeq\left\{\vx\bmomega~:~\vx\in H, \bmomega\in\alphabeteos^\infty\right\}
\end{equation}
\end{definition}
In essence, a cylinder set of rank $k$ 
is the set of infinite strings that share their $k$-prefix with some string $\vx \in H\subseteq\alphabeteos^k$.
For a length-$k$ string $\bm{x}=x_1\dotsm x_k$, the rank-$k$ cylinder set $\overC(\bm{x})\defeq\overC(\{\bm{x}\})$ is the set of all infinite strings prefixed by $\bm{x}$.\footnote{\label{fn:thin-cylinder}This type of cylinder set, i.e., one that is generated by a singleton, is also called a \defn{thin cylinder}.} We denote the collection of all rank-$k$ cylinder sets by $\overcalC_k\defeq\left\{\overline{C}(H)~:~H\in\powerset{\alphabeteos^k}\right\}$ and define $\overcalC \defeq \bigcup_{k=1}^\infty\overcalC_k$ to be the collection of all cylinder sets over $\Omega$.\footnote{Observe that $\overcalC_1\subset\overcalC_2\subset\overcalC_3\subset\dotsm$.} 
\begin{restatable}{lemma}{eosCylinderIsAlgebra} \label{lem:cylinder-algebra}
$\overcalC\subset\powerset\Omega$ is an algebra over $\Omega=\alphabeteos^\infty$.
\end{restatable}
\begin{proof}
See \cref{sec:pre-measure-proof}.
\end{proof}
 We are now ready to define the pre-measure $\P_0$ for the cylinder algebra $\overcalC$. Given an \autoregmodelAcronym $\pASM$ and any set $\overC(H)\in\overcalC$, let
\begin{align}
  \textstyle\P_0(\overC(H))\defeq\sum_{\bm{x}\in H} \pASM(\bm{x}) \label{eq:pre-measure}
\end{align}
where, denoting the length of $\vx$ by $k$,\looseness=-1
\begin{align}\label{eq:cylinder-prod}
    \textstyle\pASM(\bm{x}) \defeq \prod_{t=1}^k \pASM(x_t\mid \bm{x}_{<t}).
\end{align}
We confirm in \cref{prop:pre-measure-well-defined} that $\P_0$ is well-defined
even though the cylinder set $\overC(H)$ may also arise as $\overC(H')$ where $H'\neq H$.\footnote{For example, in the infinite coin toss model, $\overC(\seq{H})=\overC(\{\seq{HH},\seq{HT}\})$.}

\begin{restatable}{lemma}{eosCylinderPremeasure} \label{lem:cylinder-premeasure}
$\P_0$ is a pre-measure over $\overcalC$.
\end{restatable}
\begin{proof}
See \cref{sec:pre-measure-proof}.
\end{proof}

\subsection{Extension of Pre-Measure} \label{sec:extension}

We have now gathered all the ingredients needed to state Carath\'eodory's theorem.

\begin{restatable}[Carath\'eodory's Extension Theorem]{theorem}{caratheodory} \label{thm:caratheodory}
Given an algebra $\calA$ over some set $\Omega$ and a probability pre-measure $\P_0:\calA\to[0,1]$, there exists a probability space $(\Omega,\calF,\P)$ such that $\mathcal{A} \subset \calF$ and $\P|_\mathcal{A}=\P_0$. 
Furthermore, the $\sigma$-algebra $\calF$
depends only on $\calA$ and is minimal and unique---thus we may denote it by $\sigma(\calA)$---and the probability measure $\P$ is unique.
\end{restatable}
\begin{proof}[Proof Sketch]
See \cref{sec:extension-proof}.
\end{proof}

Applying Carath\'eodory's extension theorem to our cylinder algebra $\overcalC$ and pre-measure $\P_0$, we see that there exists a probability space $(\alphabeteos^\infty,\sigma(\overcalC),\P)$ over $\alphabeteos^\infty$ that agrees with the \autoregmodelAcronym $\pASM$'s probabilities.
It is a fair question to ask what kinds of sets are non-measurable under this construction; we discuss this in \Cref{sec:nonmeasurable}.
\subsection{A String-Valued Random Variable}\label{sec:rv}
Having constructed the probability space $(\alphabeteos^\infty,\sigma(\overcalC),\P)$, we now demonstrate how to use it to induce a probability space over $\kleene{\alphabet} \cup \alphabet^{\infty}$ as required by \cref{def:sequence-model}. 
We will achieve this through the use of a random variable.
\begin{definition}[random variable]\label{def:rv}
A mapping $X:\Omega\to S$ between two measurable spaces $(\Omega, \calF)$ and $(A, \calG)$ is an $(A, \calG)$-valued \defn{random variable}, or a measurable mapping, if, for all $B \in \calG$,%
\begin{equation}
    X^{-1}(B)\defeq \{\omega\in\Omega:X(\omega)\in B\} \in \mathcal{F}.
\end{equation}
\end{definition}
To construct a random variable that takes values in $\alphabet^*\cup \alphabet^\infty$, \cref{def:rv} requires us to first construct a $\sigma$-algebra over $\alphabet^*\cup\alphabet^\infty$. We will do so in a similar fashion as we constructed $(\alphabeteos^\infty,\overline{\calC})$.
Given $H\subseteq \alphabet^k$, define a rank-$k$ cylinder set in $\alphabet^*\cup\alphabet^\infty$ to be\looseness=-1
\begin{align}
    C(H)\defeq\{\bm{x}\bmomega~:~\bm{x}\in H, \bmomega \in \alphabet^*\cup\alphabet^\infty\}. \label{eq:rv-cylinder}
\end{align}
Let $\mathcal{C}_k$ be the set of all rank-$k$ cylinder sets.
Define $\calC \defeq \cup_{k=1}^\infty \calC_k$.
Then, $\sigma\left(\calC\right)$ is a $\sigma$-algebra by the same reasoning as in \cref{lem:cylinder-algebra} and \cref{thm:caratheodory}.
We can now define the random variable $X$ by\footnote{In this definition, the position $k \leq \infty$ of the first \eos---a \defn{stopping time}---is itself a random variable.}
\begin{align}\label{eq:rv}
\!\!X(\bmomega) = \begin{cases}
\bmomega_{<k} & \textbf{if }\omega_k\text{ is the first \eos in }\bmomega   \\
\bmomega & \textbf{otherwise} \text{     (}\textbf{if } \eos \notin \bmomega\text{)}
\end{cases}
\end{align}
where $\bmomega\in \alphabeteos^\infty$.
We claim that $X$ is well-defined:
\begin{restatable}{proposition}{rvMeasurable}
The function $X:(\alphabeteos^\infty,\sigma(\overcalC))\to(\alphabet^*\cup\alphabet^\infty, \sigma(\calC))$
defined in \cref{eq:rv} is a measurable mapping.
\end{restatable}
\begin{proof}
See \cref{sec:rv-proof}.
\end{proof}

Any measurable function induces a probability measure on the output space, called the pushforward measure (cf. \S2.4 in \citealp{tao2011}), given by%
\begin{align} 
    P(X\in E)\defeq\P(X^{-1}(E)).
\end{align}%
One can check that $P$, defined using $\P$,
 is indeed a probability measure on $(\alphabet^*\cup\alphabet^\infty, \sigma(\calC))$ and hence $(\alphabet^*\cup\alphabet^\infty, \sigma(\calC), P)$ is a probability space. We have therefore shown that, given any \autoregmodelAcronym, we can construct an associated sequence model as defined in \cref{def:sequence-model}.%

Under the formulation of a probability space together with a random variable, useful probability quantities arise naturally and intuitively. 
In particular, when $\vx\in\alphabet^*$ is a finite string, we have%
\begin{equation}\label{eq:prob-finite-string}
P(X=\bm{x}) \defeq P(X\in\{\bm{x}\}) = \pString(\bm{x})
\end{equation}
with the definition of $\pString$ from \cref{sec:examples}.  Additionally, as we will show in the next section, the probability of the set of infinite strings $P(X \in \alphabet^{\infty})$ is the probability of generating an infinite string.\footnote{An important detail left out in this discussion is that both the singleton set $\{\vx\}$ and $\alphabet^\infty$ need to be measurable in $(\alphabet^*\cup\alphabet^\infty, \sigma(\calC))$ for the above to make sense.
This is addressed by \cref{prop:rv-singleton-measurable} and \cref{prop:rv-infinite-measurable}.}\looseness=-1

\paragraph{Deriving $\eos$}
As an aside, the preceding section allows us to motivate the $\eos$ token in \autoregmodelAcronym as a construct that emerges naturally.
Specifically, for any $\vx \in \alphabet^\ast$, rearranging \cref{eq:prob-finite-string}:
\begin{subequations}
\begin{align}
\pASM(\eos \mid \vx) &= \textstyle \frac{P(X = \vx)}{\pASM(\vx)} = \textstyle \frac{P(X = \vx)}{P(X \in C({\vx}))} \\
    &= P(X = \vx \mid X \in C({\vx}))
\end{align}
\end{subequations}
where we have used $\pASM(\vx) = \P(\overC(\vx)) = \P(X^{-1}(C(\vx))) = P(X\in C(\vx))$.
This means that the $\eos$ probability in an \autoregmodelAcronym emerges as the conditional probability that, given that we must generate a string with a prefix $\vx \in \alphabet^\ast$, the string is exactly $\vx$.

\section{Characterizing Tightness} \label{sec:tightness}
Beyond the measure-theoretic formalization, a goal of this paper is to provide an exact characterization of tightness in \autoregmodelAcronym{}s. 
The results presented in this section generalize Lemma 3.2 in \citet{welleck-etal-2020-consistency}.
First, we consider the event
\begin{align}
A_k\defeq\{\bmomega\in\alphabeteos^\infty:\omega_k=\textsc{eos}\} \label{eq:term-at-k}
\end{align}
in the probability space $(\alphabeteos^\infty,\sigma(\overcalC),\P)$. Intuitively, $A_k$ is the event that an $\eos$ symbol appears at position $k$ in the string.
Note that under this definition the $A_k$ are not disjoint.
For example, the string $\bmomega = \seq{ab}\,\eos\,\seq{c}\,\eos\, \seq{dd}\dotsm$ lives in the intersection of $A_3$ and $A_5$ since $\eos$ appears at both position 3 and position 5. 
Using \cref{eq:term-at-k}, we can express the event 
consisting of all finite strings as $\bigcup_{k=1}^\infty A_k$. 
It follows that we can express the event of an infinite string as $\left(\bigcup_{k=1}^\infty A_k\right)^\mathsf{c}=\bigcap_{k=1}^\infty A_k^\mathsf{c}$. 
Thus, using the random variable $X$, we can express the probability of generating an infinite string as
\begin{subequations}
    \begin{align}
    P(X\in \alphabet^\infty)&=\P(X^{-1} (\alphabet^\infty))  \\
 &=\textstyle \P\left(\bigcap_{k=1}^\infty A^\mathsf{c}_k\right).
\end{align}
\end{subequations}
Hence, we can now formalize the notion of tightness, which we have introduced in \cref{sec:examples} and \cref{def:language-model}.\looseness=-1
\begin{definition}\label{def:tight}
A sequence model $P$ is said to be \defn{tight} if $P(X\in\alphabet^\infty) = 0$, in which case it is also a language model (cf. \cref{prop:seq-tight}).
Otherwise, we say that it is \defn{non-tight}.
\end{definition}

Note that the definition of $A_k$ only uses a string's $k$-prefix, and hence is a cylinder set of rank $k$. Recalling that the cylinder sets are measurable and so are the sets countably generated by them, we see that both the event consisting of all finite strings and the event consisting of all infinite strings are measurable.
Thus, $\P\left(\cup_{k=1}^\infty A_k\right)$ and $\P\left(\cap_{k=1}^{\infty} A_k^\mathsf{c}\right)$ are well defined.\looseness=-1

\subsection{A Lower Bound Result}
We have characterized tightness in terms of the probability of a specific event $\P\left(\cap_{k=1}^{\infty} A_k^\mathsf{c}\right)$, a quantity we now seek to determine.
\begin{restatable}{lemma}{durrettConditional} \label{lem:durrett-435}
If $\sum_{n=2}^\infty \P\left(A_n\mid \cap_{m=1}^{n-1}A_m^\mathsf{c}\right)=\infty$, then
$\P\left(\cap_{m=1}^\infty A_m^\mathsf{c}\right)=0$.
\end{restatable}
\begin{proof}
See \cref{sec:termination-proof}.
\end{proof}

Using \cref{lem:durrett-435}, we can derive the following useful sufficient condition for a sequence model derived from an ASM to be tight. It applies when the probability of \eos does not decay too rapidly with the length of the prefix.  
\begin{restatable}{proposition}{divergentLowerBound} \label{prop:div-implies-tight}
    If $\pASM(\textsc{eos}\mid\bm{x})\geq f(t)$ for all $t \geq 1, \bm{x}\in\alphabet^{t-1}$, and $\sum_{t=1}^\infty f(t)=\infty$, then $\P(\cap_{k=1}^\infty A_k^\mathsf{c})=0$. In other words, $P$ is tight.
\end{restatable}
\begin{proof}
See \cref{sec:div-implies-tight-proof}.
\end{proof}

\noindent This test implies tightness for all of the tight examples in \cref{sec:examples}, but not for the non-tight ones.  Note that the lower-bounding function $f$ depends only on the length of the prefix, not its content.  $f$ does not have to be monotonic---in the case of the even/odd example from \cref{sec:examples}, it is not.

\subsection{The Borel--Cantelli Lemmata}
It turns out that \cref{prop:div-implies-tight} admits a converse statement in which we can prove a similar property of $\pASM$ by assuming that the model is tight.
To prove this result, we will use a fundamental inequality from probability theory---the Borel--Cantelli lemmata. 
The Borel--Cantelli lemmata are useful for our purposes because they relate the probability measure of sets of the form $\bigcap_{n=0}^\infty A_n$ or $\bigcup_{n=0}^\infty A_n$ to a series $\sum_{n=0}^\infty p_n$.
We will only state the lemmata here without supplying their proofs;\footnote{See \S2.3 in \citet{durrett_2019} or \S4 in \citet{billingsley1986} instead.} however, we point out that \cref{lem:durrett-435} can be viewed as a parallel statement to the Borel--Cantelli lemmata and one can prove the lemmata using a very similar proof (cf. proof of Thm 2.3.7 in \citealp{durrett_2019}).

Concretely, given a sequence of events $\{A_n\}_{n=1}^\infty$ in some probability space, the Borel--Cantelli lemmata are statements about the event
\begin{equation} \label{eq:defn-io}
    \textstyle\{A_n \infoft\}\defeq\bigcap_{m=1}^\infty\bigcup_{n=m}^\infty A_n
\end{equation}
where i.o. stands for ``infinitely often.''
Intuitively, $\{A_n \infoft\}$ is the set of outcomes that appear in infinitely many sets in the collection $\{A_n\}_{n=1}^\infty$ (hence the name).
We will not use Borel--Cantelli directly, but they offer a probabilistic proof of a key result (\cref{cor:durrett-434}) which will in turn lead to the desired statement about tightness.
We formally state the first and second Borel--Cantelli lemmata below.\looseness=-1
\begin{lemma}[Borel--Cantelli I] \label{thm:bc1}
If \\ $\sum_{n=1}^\infty\P(A_n) < \infty$, then $\P(A_n \text{ i.o.})=0$.
\end{lemma}

\begin{lemma}[Borel--Cantelli II] \label{thm:bc2}
If \\ $\sum_{n=1}^\infty \P(A_n) = \infty$, then $\P(A_n \text{ i.o.})=1$, provided that $\{A_n\}$ is a sequence of independent events.
\end{lemma}

Using the Borel--Cantelli lemmata, we can prove the following useful fact.
\begin{restatable}{corollary}{durrettSequence}
\label{cor:durrett-434}
Given a sequence $\{p_n\}$ where $p_n\in[0,1)$. Then,
\begin{equation}
    \textstyle\prod_{n=1}^\infty (1-p_n) = 0 \Longleftrightarrow \sum_{n=1}^\infty p_n=\infty.
\end{equation}
\end{restatable}
\begin{proof}
See \cref{sec:durrett-434-proof}.
\end{proof}

We now turn to proving a more general version of \cref{prop:div-implies-tight}, which would imply its converse. First, we define the following quantity
\begin{align}\label{eq:ptildeeos-set}
\ptildeEOS(t)\defeq\P(A_t\mid A_1^\mathsf{c}\cap \dotsm \cap A_{t-1}^\mathsf{c})
\end{align}
which can be viewed as the \eos probability at step $t$, given that \eos was not generated at any earlier step. In \cref{eq:weighted-eos} in \cref{sec:div-implies-tight-proof}, we show that, when $\ptildeEOS(t)$ is defined, it has the same value as
\begin{align}\label{eq:ptildeeos}
\ptildeEOS(t)=\frac{
    \sum_{\bmomega\in \alphabet^{t-1}} \pASM(\bmomega)\pASM(\textsc{eos}\mid\bmomega)}
    {\sum_{\bmomega\in \alphabet^{t-1}} \pASM(\bmomega)\phantom{\pASM(\textsc{eos}\mid\bmomega)}}.
\end{align}
We can now completely characterize the tightness of an \autoregmodelAcronym with the following theorem.\looseness=-1 
\begin{restatable}[Proposition 2.4 in \citealp{meister-tacl2022}]{theorem}{lmTightMain} \label{prop:lm-tight-main}
An \autoregmodelAcronym is tight if and only if $\ptildeEOS(t)=1$ for some $t$ or $\sum_{t=1}^\infty \ptildeEOS(t)=\infty$.
\end{restatable}
\begin{proof}
See \cref{sec:lm-tight-main-proof}.  The proof uses \cref{cor:durrett-434}, which accounts for the form of the condition.
\end{proof}
We remark that \cref{prop:lm-tight-main} is a generalization of \cref{prop:div-implies-tight} since if $\ptildeEOS(t)$ is lower-bounded by $f(t)$ whose series diverges, its own series would also diverge. 
However, since $\ptildeEOS(t)$ involves the computation of a partition function in its denominator, it can be intractable to calculate \citep{lin-etal-2021-limitations}. Hence, \cref{prop:div-implies-tight} will be our main tool for determining tightness.

Finally, we note that \cref{prop:lm-tight-main} generalizes claims in previous work.
For example, \citet{welleck-etal-2020-consistency} require $f(t) = c > 0$ for some constant $c$ to determine tightness.
Hence, their bound is not helpful in determining the tightness in either \cref{ex:non-tight-rnn} or \cref{ex:tight-rnn}, because the \eos probability can be arbitrarily small in both cases. 
Applying \cref{prop:div-implies-tight}, we see that (1) the \autoregmodelAcronym in \cref{ex:non-tight-rnn} is non-tight, because the series $\sum_{t=1}^\infty \frac{1}{e^t+1}$ is convergent, and (2) the \autoregmodelAcronym in \cref{ex:tight-rnn} is tight, since the series $\sum_{t=1}^\infty \frac{1}{t+1}$ is divergent.

\section{Analysis of Common Language Models} \label{sec:lm-analysis}

We now put into practice the foundations built up in the previous sections and discuss the tightness of several classes of \autoregmodelAcronym{}s.

\subsection{Stochastic Finite-State Language Models}\label{sec:sfssm}

Language modeling based on $n$-grams has been historically 
influential in NLP \citep[Ch. 4]{jurafsky2009}.
However, as \cref{fig:sfssm} illustrates, $n$-gram language models are specific cases
of the more general stochastic finite-state language models \citep{pfsa}.
Tightness is more naturally characterized in this more general setting, as it turns out.
We begin with a linear-algebraic definition of stochastic finite-state language models---or, more precisely, sequence models, since in this paper we do not consider the non-tight ones to be language models.

\begin{definition} \label{def:sfslm}
A $Q$-state \defn{stochastic finite-state sequence model} (SFSSM) is a quadruple $\big(\alphabet, \vsource, \{\mP^{(a)}\}_{a \in \alphabet}, \vtarget\big)$, where $\alphabet$ is an alphabet of symbols, $\mP^{(a)} \in \R^{Q \times Q}_{\ge 0}$ is a symbol-specific transition matrix for $a \in \alphabet$,\footnote{For simplicity, we have disallowed $\varepsilon$-transitions.} $\vsource \in \R^{Q}_{\ge 0}$ is a vector of initial state probabilities, and $\vtarget \in \R^{Q}_{\ge 0}$ is a vector of termination probabilities, i.e., probabilities of generating $\eos$ in each state.\footnote{We use $Q$ to denote the number of states as $Q$ is the traditional notation for the set of states in a finite-state automaton.\looseness=-1
}
We further require that $\sum_{q=1}^Q s_q=1$ and that $t_q + \sum_{q'=1}^Q \mP_{qq'} = 1$
for all $1 \leq q \leq Q$, where $\mP \defeq \sum_{a \in \alphabet} \mP^{(a)}$ is the \defn{transition sum matrix}.
\end{definition}

Given an SFSSM $\Big(\alphabet, \vsource, \{\mP^{(a)}\}_{a \in \alphabet}, \vtarget\Big)$, the probability of a string
$\xx \in \kleene{\alphabet}$ is defined by%
\begin{equation}\label{eq:sfssm-string-prob}
\pASM(x_1 \cdots x_n) = \textstyle \vsource^{\top} \left( \prod_{t=1}^n \mP^{(x_t)} \right) \vtarget.
\end{equation}
\begin{definition} \label{def:acc-co-acc}
A state $q$ of an SFSSM ($1 \leq q \leq Q$) is \defn{accessible} if there is a positive-probability path to $q$ from some state $r$ with $s_r>0$; it is \defn{co-accessible} if there is a positive-probability path from $q$ to some state $r$ with $t_r>0$.
It is \defn{useful} if it is both accessible and co-accessible, i.e., $q$ appears on some positive-probability accepting path.
\end{definition}
\noindent\cref{def:acc-co-acc} allows a simple characterization of tight SFSSMs, namely \cref{thm:sfslm-tight}, and a straightforward proof of \cref{cor:ngram-mle-tight}.\footnote{\Cref{cor:ngram-mle-tight} is a special case of \citet{chi-geman-1998-estimation}, who showed that MLE estimates of PCFGs are tight.}
\begin{restatable}{theorem}{sfslmTight} \label{thm:sfslm-tight}
An SFSSM is tight iff all accessible states are also co-accessible.%
\end{restatable}
\begin{proof}
See \cref{sec:sfslm-proof}.
\end{proof}
\begin{restatable}{corollary}{ngramMLETight}
\label{cor:ngram-mle-tight}
Maximum likelihood estimates of $n$-gram models based on some corpus are tight.
\end{restatable}
\begin{proof}
See \cref{sec:sfslm-proof}.
\end{proof}

In fact, we can express the termination probability of an SFSSM in simple linear algebra terms.\looseness=-1
\begin{definition} \label{def:trim}
\defn{Trimming} an SFSSM means removing its non-useful (useless) states to obtain a \defn{substochastic finite-state sequence model}.\footnotemark{} This does not affect the string probabilities \labelcref{eq:sfssm-string-prob}.  Removing the non-useful states means removing their rows and columns from $\mP$ as well as their rows from $\vsource$ and $\vtarget$, yielding possibly smaller $\mP', \vsource'$ and $\vtarget'$.\looseness=-1

\end{definition} 
\footnotetext{We use the term \emph{substochastic} rather than \emph{stochastic} here because the trimmed model satisfies $t'_q + \sum_{q'=1}^{Q'} \mP'_{qq'} \leq 1$, but might no longer achieve equality as required by \cref{def:sfslm}.}
\begin{restatable}{theorem}{trimSfslmTight} \label{thm:sub-fslm-tight}
Let $\mP'$ be the transition sum matrix of a trimmed substochastic FSSM.  Then $\mI-\mP'$ is invertible and $P(X \in \alphabet^*)=\vsource'^\top (\mI-\mP')^{-1} \vtarget' \leq 1$.

\end{restatable}
\begin{proof}
See \cref{sec:subfslm-proof}.
\end{proof}
The well-known matrix inversion formula used above finds the total weight of all accepting paths in any weighted graph \cite{Tarj81b}.\footnote{This is assuming the total weight is finite (which we guarantee by substochasticity) and the matrix is invertible (which we guarantee by trimming)}  The formula can be seen as a special case of \citeposs{lehmann1977algebraic} algebraic path algorithm. 

\subsection{Transformer Language Models}\label{sec:transformer-thms}
We now prove that all Transformer language models are tight.
Key to our proof of the tightness of various neural architectures, including the Transformer, is the following basic fact in topology.
\begin{restatable}{theorem}{compact}\label{thm:compact}
Let $X$ be a compact topological space and $Y$ be any topological space. If $f:X\to Y$ is continuous, then $f(X)\subseteq Y$ is also compact.
\end{restatable}
\begin{proof}
See \cref{sec:transformer-proofs}.
\end{proof}
To address the variable-length nature of modern deep NLP models, we will mathematically abstract them as a function on vector  tuples,\footnote{Here $\left(\R^d\right)^+$ is the set of nonempty tuples of vectors in $\R^d$.  This is formally the disjoint union (coproduct) $\coprod_{t\in\Z_{>0}}\R^{t\times d}$.} $f: \left(\R^d\right)^+ \to \left(\R^d\right)^+$, that is length-preserving in the sense that $f\left(\R^{t\times d}\right)\subseteq \left(\R^{t\times d}\right)$ for all $t > 0$.
Intuitively, this definition is saying that $f$ is a function that maps a nonempty vector tuple $\{\bm{v}_i\}_{i=1}^t$ 
to another vector tuple $\{\bm{h}_i\}_{i=1}^t$ of the same length, 
\begin{align}\label{eq:transformer}
f(\vv_1,\dots,\vv_t)=(\vh_1,\dots,\vh_t)\in\R^{t\times d},
\end{align}
where $\vv_i \in \R^d$ is commonly the embedding of the input symbol $x_i$.
In particular, we can take the function $f: \left(\R^d\right)^+ \to \left(\R^d\right)^+$ to be the function defined by a stack of Transformer layers. This setup will help us state the following.

\begin{restatable}{lemma}{transformerCompact} \label{lem:transformer-compact}
Let $f: \left(\R^d\right)^+ \to \left(\R^d\right)^+$ be the function defined by a finite number of Transformer layers (e.g., $n$ layers) with any continuous activation function. Given a compact set $K\subset \R^d$. Then, there exists a compact set $K'\subset\R^d$ such that for every $t\in\Z_{>0}$,
\begin{equation}
    f\big(K^t\big)\subseteq \left(K'\right)^t.
\end{equation}
\end{restatable}
\begin{proof}
See \cref{sec:transformer-proofs}.
\end{proof}

Recall that a Transformer language model---or more precisely, a Transformer ASM---defines the conditional probabilities using the softmax transformation%
\begin{align}\label{eq:transformer-softmax}
\pASM(x_{t+1}\mid \vx_{\leq t})=\frac{\exp(\vu_{x_{t+1}}^\top \vh_t)}{\sum_{y\in\alphabeteos}\exp(\vu_y^\top \vh_t)}
\end{align}
where $\vu_x\in\R^d$ is the output symbol embedding of $x\in\alphabeteos$ and $\vh_t$ is defined from the input embeddings of $\vx_{\leq t}$ via \cref{eq:transformer}.  
Using \cref{lem:transformer-compact}, together with the finiteness of the vocabulary $\alphabet$ and the continuity of the softmax transformation \labelcref{eq:transformer-softmax}, readily yields our main result on Transformers.

\begin{restatable}{theorem}{transformerMain} \label{thm:transformer-main}
The autoregressive sequence model defined by any (fixed-depth) Transformer is tight.
\end{restatable}
\begin{proof}
See \cref{sec:transformer-proofs}.
\end{proof}

\subsection{Recurrent Neural Language Models}\label{sec:rnn-thms}%
Recall that the hidden state of an RNN is typically defined by the recurrence
\begin{align}
\vh_t=\sigma\left(
\bm{W}\vv_t+\bm{U}\vh_{t-1}+\bm{b}
\right)
\end{align}
where $\vv_t \in \R^d$ is the embedding of the input symbol $x_t$, as above, and $\sigma(\cdot)$ is some activation function \cite{elman1990}. The conditional probabilities are usually defined in the same way as \cref{eq:transformer-softmax}. Using \cref{thm:compact} and the same strategy of proof as in \cref{thm:transformer-main}, one can also easily prove the tightness of any RNN ASM with bounded activations (e.g., $\tanh$ or $\mathrm{sigmoid}$). However, as we saw in \cref{ex:non-tight-rnn}, an unbounded activation function (e.g., $\relu$) can indeed lead to non-tightness by making the probability of \eos decay too fast. The condition derived in \cref{prop:lm-tight-main} precisely determines how fast such decay can be without losing the tightness of the language model. Below, we generalize this result as well as Lemma 3.2 of \citet{welleck-etal-2020-consistency}, and show that if the norm of the activations eventually grows sub-logarithmically, the RNN is still tight.

\begin{restatable}{proposition}{rnnTight}\label{thm:rnn-tight}
Given an RNN ASM over $\alphabeteos$. 
Again let the output symbol vector  be $\vu_{x} \in \R^d$ for $x \in \alphabeteos$, and set $k \defeq \sup_{x \in \alphabet} \Vert\vu_x- \ueos\Vert_2$. 
Additionally, for each $t > 0$, let  $\Vert\widehat\vh_t\Vert_2$ be the maximum attainable hidden state norm for any context $\vx \in \alphabet^t$.
Such a sequence model is tight if $k \Vert\widehat\vh_t\Vert_2 \le \log t$ for all sufficiently large $t$.
\end{restatable}

\begin{proof}
See \cref{app:rnn-tight}.
\end{proof}

This result is weaker than \cref{thm:transformer-main} because in an RNN, unlike a Transformer, the depth of the computation graph grows with the sequence length.

\section{Conclusion}%
This paper presents a measure-theoretic treatment of language modeling and its tightness. 
Practical implications of our results include determining when sampling from an autoregressive sequence model is guaranteed to terminate and whether MCMC algorithms over such models will mix to the correct distribution. 

To this end, we first defined various components of language modeling in measure-theoretic terminology. 
This in turn allows us to understand the portion of probability mass allocated to  infinite-length strings. 
Importantly, this presentation formalizes a definition of sequence modeling under which the probability of producing an infinite-length sequence is non-zero; while today's models are often capable of producing such strings, previously there was no rigorous treatment of this case. 

Indeed, such a definition is useful when considering a number of neural architectures (e.g., a simple RNN as in \citealp{elman1990}) and language generation systems (e.g., the distribution induced by nucleus sampling; \citealp{Holtzman2020The}). 
In particular, we showed that perhaps the most commonly-used NLP architecture, the Transformer language model, is indeed a language model---a tight distribution over finite strings---a property that had been called into question by previous work.

\section*{Limitations}
Our discussion in this paper leaves out the consideration of computability of measures over languages.
Specifically, we note that there exist works on computable measure theory developed in the context of theoretical computer science \citep{de-leeuw-1956} and probabilistic programming languages \citep{roy-2011-thesis}.
Additional machinery needs to be developed in order for a proper treatment and we leave this for future work.

Another notable limitation is that we exclusively focused on the autoregressive production of language. 
Importantly, our formalism might not be compatible with other models of language production such as those induced by a PCFG.

Finally, our proofs of \cref{thm:transformer-main,thm:rnn-tight} exploit the strictly positive property of the softmax function. Importantly, they do not apply to models with sparse distributions \cite{pmlr-v48-martins16,peters-etal-2019-sparse,martins-reconciling-2021}.

\section*{Ethics}

There are no ethical implications of this paper to the best knowledge of the authors.

\section*{Acknowledgments}
We thank Holden Lee for helpful discussion and suggestions, Anej Svete for help with the graphics and Chu-Cheng Lin and Xinyan Velocity Yu for helpful comments. LD is partially supported by the Johns Hopkins Mathematical Institute for Data Science (MINDS) Fellowship. Finally, we thank the students of the \href{https://rycolab.io/classes/llm-s23/}{LLM course at ETH Z\"urich (263-5354-00L)} for carefully reading this paper as part of their lecture notes, and in particular, Valentin Bieri for making a valuable remark.

\bibliography{arxiv-ready}

\begin{thebibliography}{49}
\expandafter\ifx\csname natexlab\endcsname\relax\def\natexlab#1{#1}\fi

\bibitem[{Axler(2020)}]{axler2020}
Sheldon Axler. 2020.
\newblock \href {https://doi.org/10.1007/978-3-030-33143-6} {\emph{Measure,
  Integration {\&} Real Analysis}}.
\newblock Springer International Publishing.

\bibitem[{Billingsley(1995)}]{billingsley1986}
Patrick Billingsley. 1995.
\newblock \href
  {https://www.wiley.com/en-us/Probability+and+Measure\%2C+Anniversary+Edition-p-9781118122372}
  {\emph{Probability and Measure}}, 3$^{\text{rd}}$ edition.
\newblock Wiley.

\bibitem[{Blackwell and Diaconis(1996)}]{blackwell-1996}
David Blackwell and Persi Diaconis. 1996.
\newblock \href
  {https://projecteuclid.org/ebooks/institute-of-mathematical-statistics-lecture-notes-monograph-series/Statistics-probability-and-game-theory/Chapter/A-non-measurable-tail-set/10.1214/lnms/1215453560}
  {A non-measurable tail set}.
\newblock \emph{Statistics, probability and game theory: Papers in honor of
  David Blackwell}, 30:1--5.

\bibitem[{Booth and Thompson(1973)}]{booth1973}
Taylor~L. Booth and Richard~A. Thompson. 1973.
\newblock \href {https://doi.org/10.1109/T-C.1973.223746} {Applying probability
  measures to abstract languages}.
\newblock \emph{IEEE Transactions on Computers}, C-22(5):442--450.

\bibitem[{Bostrom and Durrett(2020)}]{durrett-nlp}
Kaj Bostrom and Greg Durrett. 2020.
\newblock \href {https://doi.org/10.18653/v1/2020.findings-emnlp.414} {Byte
  pair encoding is suboptimal for language model pretraining}.
\newblock In \emph{Findings of the Association for Computational Linguistics:
  EMNLP 2020}, pages 4617--4624, Online. Association for Computational
  Linguistics.

\bibitem[{Brown et~al.(2020)Brown, Mann, Ryder, Subbiah, Kaplan, Dhariwal,
  Neelakantan, Shyam, Sastry, Askell, Agarwal, Herbert-Voss, Krueger, Henighan,
  Child, Ramesh, Ziegler, Wu, Winter, Hesse, Chen, Sigler, Litwin, Gray, Chess,
  Clark, Berner, McCandlish, Radford, Sutskever, and Amodei}]{gpt3}
Tom Brown, Benjamin Mann, Nick Ryder, Melanie Subbiah, Jared Kaplan, Prafulla
  Dhariwal, Arvind Neelakantan, Pranav Shyam, Girish Sastry, Amanda Askell,
  Sandhini Agarwal, Ariel Herbert-Voss, Gretchen Krueger, Tom Henighan, Rewon
  Child, Aditya Ramesh, Daniel Ziegler, Jeffrey Wu, Clemens Winter, Chris
  Hesse, Mark Chen, Eric Sigler, Mateusz Litwin, Scott Gray, Benjamin Chess,
  Jack Clark, Christopher Berner, Sam McCandlish, Alec Radford, Ilya Sutskever,
  and Dario Amodei. 2020.
\newblock \href
  {https://proceedings.neurips.cc/paper/2020/file/1457c0d6bfcb4967418bfb8ac142f64a-Paper.pdf}
  {Language models are few-shot learners}.
\newblock In \emph{Advances in Neural Information Processing Systems},
  volume~33, pages 1877--1901.

\bibitem[{Chen et~al.(2018)Chen, Gilroy, Maletti, May, and
  Knight}]{chen-etal-2018-recurrent}
Yining Chen, Sorcha Gilroy, Andreas Maletti, Jonathan May, and Kevin Knight.
  2018.
\newblock \href {https://doi.org/10.18653/v1/N18-1205} {Recurrent neural
  networks as weighted language recognizers}.
\newblock In \emph{Proceedings of the 2018 Conference of the North {A}merican
  Chapter of the Association for Computational Linguistics: Human Language
  Technologies, Volume 1 (Long Papers)}, pages 2261--2271, New Orleans,
  Louisiana. Association for Computational Linguistics.

\bibitem[{Chi(1999)}]{chi-1999-statistical}
Zhiyi Chi. 1999.
\newblock \href {https://aclanthology.org/J99-1004} {Statistical properties of
  probabilistic context-free grammars}.
\newblock \emph{Computational Linguistics}, 25(1):131--160.

\bibitem[{Chi and Geman(1998)}]{chi-geman-1998-estimation}
Zhiyi Chi and Stuart Geman. 1998.
\newblock \href {https://aclanthology.org/J98-2005} {Estimation of
  probabilistic context-free grammars}.
\newblock \emph{Computational Linguistics}, 24(2):299--305.

\bibitem[{Cohen and Johnson(2013)}]{cohen-johnson-2013-effect}
Shay~B. Cohen and Mark Johnson. 2013.
\newblock \href {https://aclanthology.org/P13-1102} {The effect of
  non-tightness on {B}ayesian estimation of {PCFG}s}.
\newblock In \emph{Proceedings of the 51st Annual Meeting of the Association
  for Computational Linguistics (Volume 1: Long Papers)}, pages 1033--1041,
  Sofia, Bulgaria. Association for Computational Linguistics.

\bibitem[{de~Leeuw et~al.(1956)de~Leeuw, Moore, Shannon, and
  Shapiro}]{de-leeuw-1956}
K.~de~Leeuw, E.~F. Moore, C.~E. Shannon, and N.~Shapiro. 1956.
\newblock \href {http://www.jstor.org/stable/j.ctt1bgzb3s.12}
  {\emph{COMPUTABILITY BY PROBABILISTIC MACHINES}}, Annals of Mathematics.
  Studies, no. 34, pages 183--212. Princeton University Press.

\bibitem[{Durrett(2019)}]{durrett_2019}
Rick Durrett. 2019.
\newblock \href {https://doi.org/10.1017/9781108591034} {\emph{Probability:
  Theory and Examples}}, 5$^{\text{th}}$ edition.
\newblock Cambridge Series in Statistical and Probabilistic Mathematics.
  Cambridge University Press.

\bibitem[{Elman(1990)}]{elman1990}
Jeffrey~L. Elman. 1990.
\newblock \href
  {https://www.sciencedirect.com/science/article/abs/pii/036402139090002E}
  {Finding structure in time}.
\newblock \emph{Cognitive Science}, 14(2):179--211.

\bibitem[{Folland(1999)}]{folland-1999}
Gerald~B. Folland. 1999.
\newblock \href
  {https://www.wiley.com/en-us/Real+Analysis:+Modern+Techniques+and+Their+Applications,+2nd+Edition-p-9780471317166}
  {\emph{Real Analysis: Modern Techniques and Their Applications}},
  2$^{\text{nd}}$ edition.
\newblock Wiley.

\bibitem[{Grinstead and Snell(1997)}]{grinstead1997}
Charles~M. Grinstead and J.~Laurie Snell. 1997.
\newblock \href {https://math.dartmouth.edu/~prob/prob/prob.pdf}
  {\emph{Introduction to Probability}}, 2$^{\text{nd}}$ revised edition.
\newblock American Mathematical Society.

\bibitem[{Hale(2001)}]{hale2001probabilistic}
John Hale. 2001.
\newblock \href {https://aclanthology.org/N01-1021} {A probabilistic {E}arley
  parser as a psycholinguistic model}.
\newblock In \emph{Second Meeting of the North {A}merican Chapter of the
  Association for Computational Linguistics}.

\bibitem[{Holtzman et~al.(2020)Holtzman, Buys, Du, Forbes, and
  Choi}]{Holtzman2020The}
Ari Holtzman, Jan Buys, Li~Du, Maxwell Forbes, and Yejin Choi. 2020.
\newblock \href {https://openreview.net/forum?id=rygGQyrFvH} {The curious case
  of neural text degeneration}.
\newblock In \emph{International Conference on Learning Representations}.

\bibitem[{Horn and Johnson(2012)}]{horn2013}
Roger~A. Horn and Charles~R. Johnson. 2012.
\newblock \href {https://doi.org/10.1017/CBO9781139020411} {\emph{Matrix
  Analysis}}, 2$^{\text{nd}}$ edition.
\newblock Cambridge University Press.

\bibitem[{Jelinek(1976)}]{jelinek1976}
Frederick Jelinek. 1976.
\newblock \href {https://doi.org/10.1109/PROC.1976.10159} {Continuous speech
  recognition by statistical methods}.
\newblock \emph{Proceedings of the IEEE}, 64(4):532--556.

\bibitem[{Jurafsky and Martin(2009)}]{jurafsky2009}
Dan Jurafsky and James Martin. 2009.
\newblock \href {https://home.cs.colorado.edu/~martin/slp.html} {\emph{Speech
  and Language Processing: An Introduction to Natural Language Processing,
  Computational Linguistics, and Speech Recognition}}, 2$^{\text{nd}}$ edition.
\newblock Pearson Prentice Hall.

\bibitem[{Kolmogorov(1933)}]{kolmogorov1933}
A.~N. Kolmogorov. 1933.
\newblock \href {https://link.springer.com/book/10.1007/978-3-642-49888-6}
  {\emph{Grundbegriffe der Wahrscheinlichkeitsrechnung}}.
\newblock Springer.

\bibitem[{Lehmann(1977)}]{lehmann1977algebraic}
Daniel~J. Lehmann. 1977.
\newblock \href
  {https://www.sciencedirect.com/science/article/pii/0304397577900561}
  {Algebraic structures for transitive closure}.
\newblock \emph{Theoretical Computer Science}, 4(1):59--76.

\bibitem[{Lin(2022)}]{lin2022thesis}
Chu-Cheng Lin. 2022.
\newblock \emph{On Expressiveness, Inference, and Parameter Estimation of
  Discrete Sequence Models}.
\newblock Ph.D. thesis, Johns Hopkins University.

\bibitem[{Lin et~al.(2021)Lin, Jaech, Li, Gormley, and
  Eisner}]{lin-etal-2021-limitations}
Chu-Cheng Lin, Aaron Jaech, Xin Li, Matthew~R. Gormley, and Jason Eisner. 2021.
\newblock \href {https://doi.org/10.18653/v1/2021.naacl-main.405} {Limitations
  of autoregressive models and their alternatives}.
\newblock In \emph{Proceedings of the 2021 Conference of the North American
  Chapter of the Association for Computational Linguistics: Human Language
  Technologies}, pages 5147--5173, Online. Association for Computational
  Linguistics.

\bibitem[{Lin and McCarthy(2022)}]{lin2022on}
Chu-Cheng Lin and Arya~D. McCarthy. 2022.
\newblock \href {https://openreview.net/forum?id=SsPCtEY6yCl} {On the
  uncomputability of partition functions in energy-based sequence models}.
\newblock In \emph{International Conference on Learning Representations}.

\bibitem[{Liu et~al.(2022)Liu, Jiang, Monath, Cotterell, and
  Sachan}]{liu+al.emnlp22}
Tianyu Liu, Yuchen Jiang, Nicholas Monath, Ryan Cotterell, and Mrinmaya Sachan.
  2022.
\newblock \href {https://arxiv.org/pdf/2210.14698} {Autoregressive structure
  prediction with language models}.
\newblock In \emph{Findings of the Association for Computational Linguistics:
  EMNL 2022}, Abu Dhabi, United Arab Emirates. Association for Computational
  Linguistics.

\bibitem[{Martins and Astudillo(2016)}]{pmlr-v48-martins16}
Andre Martins and Ramon Astudillo. 2016.
\newblock \href {https://proceedings.mlr.press/v48/martins16.html} {From
  softmax to sparsemax: A sparse model of attention and multi-label
  classification}.
\newblock In \emph{Proceedings of The 33rd International Conference on Machine
  Learning}, volume~48 of \emph{Proceedings of Machine Learning Research},
  pages 1614--1623, New York, New York, USA. PMLR.

\bibitem[{Martins(2021)}]{martins-reconciling-2021}
Andr{\'{e}} F.~T. Martins. 2021.
\newblock \href {http://arxiv.org/abs/2104.00755} {Reconciling the
  discrete-continuous divide: Towards a mathematical theory of sparse
  communication}.
\newblock \emph{CoRR}, abs/2104.00755.

\bibitem[{Meister et~al.(2021)Meister, Pimentel, Haller, J{\"a}ger, Cotterell,
  and Levy}]{meister-etal-2021-revisiting}
Clara Meister, Tiago Pimentel, Patrick Haller, Lena J{\"a}ger, Ryan Cotterell,
  and Roger Levy. 2021.
\newblock \href {https://doi.org/10.18653/v1/2021.emnlp-main.74} {Revisiting
  the uniform information density hypothesis}.
\newblock In \emph{Proceedings of the 2021 Conference on Empirical Methods in
  Natural Language Processing}, pages 963--980, Online and Punta Cana,
  Dominican Republic. Association for Computational Linguistics.

\bibitem[{Meister et~al.(2022)Meister, Pimentel, Wiher, and
  Cotterell}]{meister-tacl2022}
Clara Meister, Tiago Pimentel, Gian Wiher, and Ryan Cotterell. 2022.
\newblock \href {https://arxiv.org/abs/2202.00666} {Locally typical sampling}.
\newblock \emph{Transactions of the Association for Computational Linguistics}.

\bibitem[{Munkres(2000)}]{munkres2000}
James~R. Munkres. 2000.
\newblock \href
  {https://www.amazon.com/Topology-2nd-James-Munkres/dp/0131816292}
  {\emph{Topology}}, 2$^{\text{nd}}$ edition.
\newblock Prentice Hall, Inc.

\bibitem[{Nair and Hinton(2010)}]{relu}
Vinod Nair and Geoffrey~E. Hinton. 2010.
\newblock \href {https://dl.acm.org/doi/10.5555/3104322.3104425} {Rectified
  linear units improve restricted {B}oltzmann machines}.
\newblock In \emph{Proceedings of the 27th International Conference on
  International Conference on Machine Learning}, pages 807--814, Madison, WI,
  USA.

\bibitem[{Nederhof and Satta(2006)}]{nederhof-satta-2006-estimation}
Mark-Jan Nederhof and Giorgio Satta. 2006.
\newblock \href {https://aclanthology.org/N06-1044} {Estimation of consistent
  probabilistic context-free grammars}.
\newblock In \emph{Proceedings of the Human Language Technology Conference of
  the {NAACL}, Main Conference}, pages 343--350, New York City, USA.
  Association for Computational Linguistics.

\bibitem[{Oxtoby(1980)}]{Oxtoby1980}
John~C. Oxtoby. 1980.
\newblock \href {https://doi.org/10.1007/978-1-4684-9339-9} {\emph{Measure and
  Category: A Survey of the Analogies between Topological and Measure Spaces}}.
\newblock Springer New York.

\bibitem[{Peters et~al.(2019)Peters, Niculae, and
  Martins}]{peters-etal-2019-sparse}
Ben Peters, Vlad Niculae, and Andr{\'e} F.~T. Martins. 2019.
\newblock \href {https://doi.org/10.18653/v1/P19-1146} {Sparse
  sequence-to-sequence models}.
\newblock In \emph{Proceedings of the 57th Annual Meeting of the Association
  for Computational Linguistics}, pages 1504--1519, Florence, Italy.
  Association for Computational Linguistics.

\bibitem[{Peters et~al.(2018)Peters, Neumann, Iyyer, Gardner, Clark, Lee, and
  Zettlemoyer}]{peters-etal-2018-deep}
Matthew~E. Peters, Mark Neumann, Mohit Iyyer, Matt Gardner, Christopher Clark,
  Kenton Lee, and Luke Zettlemoyer. 2018.
\newblock \href {https://doi.org/10.18653/v1/N18-1202} {Deep contextualized
  word representations}.
\newblock In \emph{Proceedings of the 2018 Conference of the North {A}merican
  Chapter of the Association for Computational Linguistics: Human Language
  Technologies, Volume 1 (Long Papers)}, pages 2227--2237, New Orleans,
  Louisiana. Association for Computational Linguistics.

\bibitem[{Petroni et~al.(2019)Petroni, Rockt{\"a}schel, Riedel, Lewis, Bakhtin,
  Wu, and Miller}]{petroni-etal-2019-language}
Fabio Petroni, Tim Rockt{\"a}schel, Sebastian Riedel, Patrick Lewis, Anton
  Bakhtin, Yuxiang Wu, and Alexander Miller. 2019.
\newblock \href {https://doi.org/10.18653/v1/D19-1250} {Language models as
  knowledge bases?}
\newblock In \emph{Proceedings of the 2019 Conference on Empirical Methods in
  Natural Language Processing and the 9th International Joint Conference on
  Natural Language Processing (EMNLP-IJCNLP)}, pages 2463--2473, Hong Kong,
  China. Association for Computational Linguistics.

\bibitem[{Reif et~al.(2022)Reif, Ippolito, Yuan, Coenen, Callison-Burch, and
  Wei}]{reif-etal-2022-recipe}
Emily Reif, Daphne Ippolito, Ann Yuan, Andy Coenen, Chris Callison-Burch, and
  Jason Wei. 2022.
\newblock \href {https://doi.org/10.18653/v1/2022.acl-short.94} {A recipe for
  arbitrary text style transfer with large language models}.
\newblock In \emph{Proceedings of the 60th Annual Meeting of the Association
  for Computational Linguistics (Volume 2: Short Papers)}, pages 837--848,
  Dublin, Ireland. Association for Computational Linguistics.

\bibitem[{Roy(2011)}]{roy-2011-thesis}
Daniel~M. Roy. 2011.
\newblock \href {http://danroy.org/papers/Roy-PHD-2011.pdf}
  {\emph{Computability, Inference and Modeling in Probabilistic Programming}}.
\newblock Ph.D. thesis, Massachusetts Institute of Technology, USA.
\newblock AAI0823858.

\bibitem[{Royden(1988)}]{royden1988real}
Halsey~L. Royden. 1988.
\newblock \href {https://books.google.com/books?id=KBCfPwAACAAJ} {\emph{Real
  Analysis}}, 3$^{\text{rd}}$ edition.
\newblock Prentice-Hall.

\bibitem[{Shannon(1948)}]{shannon1948mathematical}
Claude~E. Shannon. 1948.
\newblock \href {https://doi.org/10.1002/j.1538-7305.1948.tb01338.x} {A
  mathematical theory of communication}.
\newblock \emph{Bell System Technical Journal}, 27:623--656.

\bibitem[{Tao(2011)}]{tao2011}
Terence Tao. 2011.
\newblock \href
  {https://terrytao.files.wordpress.com/2012/12/gsm-126-tao5-measure-book.pdf}
  {\emph{An Introduction to Measure Theory}}.
\newblock American Mathematical Society.

\bibitem[{Tao(2016)}]{tao2016analysis}
Terence Tao. 2016.
\newblock \href {https://doi.org/10.1007/978-981-10-1804-6} {\emph{Analysis II:
  Third Edition}}.
\newblock Texts and Readings in Mathematics. Springer Singapore.

\bibitem[{Tarjan(1981)}]{Tarj81b}
Robert~E. Tarjan. 1981.
\newblock \href {https://dl.acm.org/doi/10.1145/322261.322273} {Fast algorithms
  for solving path problems}.
\newblock \emph{Journal of the Association for Computing Machinery},
  28(3):594--614.

\bibitem[{Ulam(1930)}]{ulam1930}
Stanis{\l}aw Ulam. 1930.
\newblock \href {http://eudml.org/doc/212487} {Zur masstheorie in der
  allgemeinen mengenlehre}.
\newblock \emph{Fundamenta Mathematicae}, 16(1):140--150.

\bibitem[{Vaswani et~al.(2017)Vaswani, Shazeer, Parmar, Uszkoreit, Jones,
  Gomez, Kaiser, and Polosukhin}]{vaswani-2017-attention}
Ashish Vaswani, Noam Shazeer, Niki Parmar, Jakob Uszkoreit, Llion Jones,
  Aidan~N. Gomez, \L{}ukasz Kaiser, and Illia Polosukhin. 2017.
\newblock \href
  {https://proceedings.neurips.cc/paper/2017/file/3f5ee243547dee91fbd053c1c4a845aa-Paper.pdf}
  {Attention is all you need}.
\newblock In \emph{Advances in Neural Information Processing Systems},
  volume~30.

\bibitem[{Vidal et~al.(2005)Vidal, Thollard, de~la Higuera, Casacuberta, and
  Carrasco}]{pfsa}
E.~Vidal, F.~Thollard, C.~de~la Higuera, F.~Casacuberta, and R.C. Carrasco.
  2005.
\newblock \href {https://doi.org/10.1109/TPAMI.2005.147} {Probabilistic
  finite-state machines - part i}.
\newblock \emph{IEEE Transactions on Pattern Analysis and Machine
  Intelligence}, 27(7):1013--1025.

\bibitem[{Weaver(1955)}]{weaver1949}
Warren Weaver. 1955.
\newblock \href
  {https://repositorio.ul.pt/bitstream/10451/10945/2/ulfl155512_tm_2.pdf}
  {Translation}.
\newblock In William~N. Locke and A.~Donald Boothe, editors, \emph{Machine
  Translation of Languages}, pages 15--23. MIT Press.
\newblock Reprinted from a memorandum written by Weaver in 1949.

\bibitem[{Welleck et~al.(2020)Welleck, Kulikov, Kim, Pang, and
  Cho}]{welleck-etal-2020-consistency}
Sean Welleck, Ilia Kulikov, Jaedeok Kim, Richard~Yuanzhe Pang, and Kyunghyun
  Cho. 2020.
\newblock \href {https://doi.org/10.18653/v1/2020.emnlp-main.448} {Consistency
  of a recurrent language model with respect to incomplete decoding}.
\newblock In \emph{Proceedings of the 2020 Conference on Empirical Methods in
  Natural Language Processing (EMNLP)}, pages 5553--5568, Online. Association
  for Computational Linguistics.

\end{thebibliography}
\bibliographystyle{acl_natbib}

\appendix
\onecolumn

\section{Related Work}

The issue of tightness has been studied extensively in the context of probabilistic context-free grammars \citep[PCFG;][]{chi-geman-1998-estimation,chi-1999-statistical,cohen-johnson-2013-effect}, although \citet{chi-1999-statistical} refers to non-tight models as \defn{improper}. 
Specifically, \citet{chi-1999-statistical} gave algorithms for determining the tightness of a PCFG by formalizing a PCFG as a branching process. 
\citet{chi-1999-statistical} further proved that any maximum-likelihood estimator yields a tight PCFG.
Several previous works study the ability of language models to place probability mass on infinite-length strings \cite{booth1973,nederhof-satta-2006-estimation,chen-etal-2018-recurrent,welleck-etal-2020-consistency}, where they refer to the non-tight language models as \defn{inconsistent}.
In some cases, this behavior can be attributed to the discrepancy between the language model itself and the distribution induced by a (possibly stochastic) decoding algorithm: the decoder may have a lower probability of generating the \eos token.  For example, on the tight bigram model of \cref{ex:tight-2-gram}, a greedy decoder will always generate $\seq{a}$ and never \eos. Yet in other examples, it is the model itself that leaks probability mass to infinite-length strings, i.e., it may be non-tight, which is the problem we focus on in this work, providing a characterization of tightness.
Notably, the conditions we propose are more general than those of \citet{welleck-etal-2020-consistency}.

Several other works consider the limitations of common neural network architectures for modeling distributions over finite sequences (strings), albeit focusing specifically on other attributes, such as their computational complexity for problems like equivalence or undecidability \cite{chen-etal-2018-recurrent,lin-etal-2021-limitations,lin2022on,lin2022thesis}. 
In contrast, this work provides a formal treatment of language models by enlarging the sample space to $\kleene{\alphabet} \cup \alphabet^\infty$, although to ensure tightness, $\alphabet^\infty$ must receive probability 0.  Such definitions are not uncommon in probability theory. For example, while the Wiener process (i.e., the standard Brownian motion) is a distribution over \emph{all} functions, the definition ensures that the set of discontinuous functions is assigned probability 0 \cite[Ch. 7]{durrett_2019}.

\citet{meister-tacl2022} similarly address the notion of a language model as a distribution over infinite sequences by casting such models as stochastic processes. They use this framing in order to motivate decoding, without providing comprehensive measure-theoretic foundations of such distributions. 

\section{Details for Motivating \cref{ex:tight-2-gram}} \label{sec:tight-2-gram-proof}

Here, we lay out the steps to calculate $\P(\alphabet^*)$ from \cref{fig:tight-2-gram-example}:
\begin{subequations}
\begin{align}
    \P(\alphabet^*)
    &= \sum_{n=0}^\infty \left(\P(\seq{a}^{n+1})+\sum_{m=0}^\infty \P(\seq{a}^{n+1}\seq{b}^{m+1}) \right) \\
    &= \sum_{n=0}^\infty \left(1\cdot (0.7)^n\cdot 0.1 + \sum_{m=0}^\infty 1\cdot (0.7)^n \cdot 0.2 \cdot (0.9)^m \cdot 0.1 \right) \\
    &= \sum_{n=0}^\infty \;(0.7)^n\cdot \left(0.1 + 0.2 \cdot \left( \sum_{m=0}^\infty (0.9)^m \right) \cdot 0.1 \right) \\
    &= \sum_{n=0}^\infty \;(0.7)^n\cdot \left(0.1 + 0.2 \cdot \frac{1}{1-0.9} \cdot 0.1 \right) \\
    &= \sum_{n=0}^\infty \;(0.7)^n\cdot 0.3 = \frac{0.3}{1-0.7} = 1
\end{align}
\end{subequations}

\section{Measure Theory Details}

\subsection{Proofs and Details in \cref{sec:pre-measure}} \label{sec:pre-measure-appendix}

\subsubsection{Details of \cref{thm:ac-ch-impossible}} \label{sec:ac-ch-impossible-proof}

\ACCHImpossible*

This theorem is an \defn{impossibility of measure} theorem. 
Generally speaking, the existence of a non-measurable set implies some form of impossibility of measure.
The most famous example of non-measurable sets are Vitali sets, which exist given the Axiom of Choice.  Vitali's 1905 construction is typically described in introductory texts on measure theory \cite{royden1988real, billingsley1986, axler2020}. 
The existence of Vitali sets shows that it is impossible to define a probability measure that satisfies translational invariance on the measurable space $\big([0,1), \mathcal{P}([0,1))\big)$.  Thus, to achieve translational invariance, Lebesgue measure uses a $\sigma$-algebra smaller than $\mathcal{P}([0,1))$, in which the Vitali sets are among the non-measurable sets.  However, the translational invariance desideratum is not relevant to our space of discrete sequences. 
A theorem by \citet{ulam1930} reveals a deeper reason that some sets must be non-measurable. We shall state the theorem below as given in \citet{Oxtoby1980} and omit its proof. We refer interested readers to Chapter 5 in \citet{Oxtoby1980}, which contains an accessible proof and an excellent discussion of the theorem including its generalizations and historical context.\looseness=-1
\begin{theorem}[\citealp{ulam1930}] \label{thm:ulam-1930}
Assuming the Axiom of Choice, a finite measure $\mu$ defined for all subsets of a set $X$ of cardinality $\aleph_1$ vanishes identically [that is, equals zero for all subsets] if it is equal to zero for every one-element subset.
\end{theorem}

In the statement above, $\aleph_1$ denotes the cardinality of the first uncountable ordinal. We can see that \cref{thm:ac-ch-impossible} is a straightforward consequence of \cref{thm:ulam-1930}.

\begin{proof}[Proof of \cref{thm:ac-ch-impossible}]
Recall that $\mathrm{card}(\{\seq{H},\seq{T}\}^\infty)=2^{\aleph_0}$. Assuming the Continuum Hypothesis, $2^{\aleph_0}=\aleph_1$, and hence by \cref{thm:ulam-1930}, such a measure is uniformly 0, and hence cannot be a probability measure.
\end{proof}

\subsubsection{Other Proofs in \cref{sec:pre-measure}} \label{sec:pre-measure-proof}

\eosCylinderIsAlgebra*
\begin{proof}
  First, $\Omega \in \overcalC$ since it is a cylinder set of rank 0 or indeed of any rank $k$: $\Omega=\overC(\alphabeteos^k) \in \overcalC_k \subset \overcalC$.  Second, $\overcalC$ is closed under complements: given a cylinder set of rank $k$, that is, $\overC(H)$ where $H\subseteq \alphabeteos^k$, its complement $\big(\overC(H)\big)^\mathsf{c}=\overC\left(\alphabeteos^k\setminus H\right)$ is also a cylinder set of rank $k$.  Finally, $\overcalC$ is closed under union: the union of cylinder sets of ranks $k_1 \leq k_2$ is a cylinder set of rank $k_2$, since both can be regarded as cylinder sets of rank $k_2$.
  Hence, $\overcalC$ is an algebra over $\Omega$.
\end{proof}

\begin{proposition} \label{prop:pre-measure-well-defined}
$\P_0$ as defined in \cref{eq:pre-measure} is a well-defined function.
\end{proposition}
\begin{proof}
  Suppose a cylinder set arises in two ways, $\overC(H_1)=\overC(H_2)$, where
  $H_1\subseteq\alphabeteos^{k_1}$ and $H_2\subseteq\alphabeteos^{k_2}$.
  We must show $\sum_{\bm{x}\in H_1} \pASM(\bm{x}) = \sum_{\bm{x'}\in H_2} \pASM(\bm{x'})$.
  Without loss of generality, assume that $k_1\leq k_2$.
  The definition of $\overC(H_2)$ (\cref{def:cylinder-set}) implies that $H_2$ consists of all length-$k_2$ prefixes of strings in $\overC(H_2)$.  But $\overC(H_2)=\overC(H_1)$, so the definition of $\overC(H_1)$ (\cref{def:cylinder-set}) implies that its length-$k_2$ prefixes are exactly the strings of the form $\vx\yy$ where $\vx\in H_1, \yy\in\alphabeteos^{k_2-k_1}$.  Hence we can write $H_2$ in terms of $H_1$ as $H_2=\{\vx\yy: \vx\in H_1, \yy\in\alphabeteos^{k_2-k_1}\}$.
Thus
\begin{align}
  \sum_{\bm{x'}\in H_2} \pASM(\bm{x'}) &= \sum_{\bm{x}\in H_1} \sum_{\bm{y} \in \alphabeteos^{k_2-k_1}} \pASM(\bm{x}\bm{y}) = \sum_{\bm{x}\in H_1} \pASM(\bm{x})
\end{align}
where the last equality is true because $\pASM$ is defined by the locally normalized product \labelcref{eq:cylinder-prod}.
\end{proof}

\eosCylinderPremeasure*
For the proof of \cref{lem:cylinder-premeasure}, we will mostly follow the proof of Thm 2.3 in \citet{billingsley1986}, with the exception of invoking the Tychonoff theorem directly. This proof depends on the following lemma, which is Example 2.10 in \citet{billingsley1986}. We repeat the statement and proof here for the reader's convenience.

\begin{lemma} \label{lem:pre-measure-continuity}
Let $\P_0$ be a finitely additive probability pre-measure over $\overcalC$ such that, given a decreasing sequence of sets $A_1\supset A_2\supset \dotsm$ in $\overcalC$ where $\bigcap_{n=1}^\infty A_n=\emptyset$, $\lim_{n\to\infty} \P_0(A_n)=0$. Then, $\P_0$ is also countably additive over $\overcalC$.
\end{lemma}
\begin{proof}
Let $\{A_n\}$ be a sequence of disjoint sets in $\overcalC$ such that $A=\bigcup_n A_n \in \overcalC$. Then, defining $B_n=\bigcup_{m>n} A_m$, we see that $B_1\supset B_2\supset \dotsm$ and $\bigcap_n B_n=\emptyset$. Notice that
\begin{equation}
    A=A_1\cup B_1=A_1\cup A_2\cup B_2 = \dotsm =A_1\cup\dotsm\cup A_n \cup B_n
\end{equation}
for any $n$ and hence by finite additivity of $\P_0$
\begin{equation}
    \P_0(A)=\P_0(A_1)+\dotsm+\P_0(A_n)+\P_0(B_n)
\end{equation}
or equivalently
\begin{equation}\label{eq:finite-countable-add}
    \P_0(A_1)+\dotsm+\P_0(A_n) = \P_0(A)-\P_0(B_n).
\end{equation}
Since, $B_n\downarrow\emptyset$ implies that $\P_0(B_n)\downarrow0$ by assumption, taking the limits on both sides of \cref{eq:finite-countable-add} yields
\begin{equation}
    \sum_{n} \P_0(A_n)=\lim_{n\to\infty} \sum_{i\leq n}\P_0(A_i)
    =\P_0(A)-\lim_{n\to\infty}\P_0(B_n) =\P_0(A)
\end{equation}
which shows countable additivity.
\end{proof}
We also recall the Tychonoff theorem.\footnote{See \S37 in \citet{munkres2000} for a detailed and well-written treatise.}
\begin{theorem}[Tychonoff] \label{thm:tychonoff}
Let $\{X_\alpha\}_{\alpha\in J}$ be an indexed family of compact topologies. Then, their product topology $\prod_{\alpha\in J} X_\alpha$ is also compact.
\end{theorem}
We can now give the proof for \cref{lem:cylinder-premeasure}.
\begin{proof}[Proof of \cref{lem:cylinder-premeasure}]
We first show that $\P_0$ is finitely additive over $\overcalC$. Let $C(H_1)$ and $C(H_2)$ be two disjoint cylinder sets. By \cref{prop:pre-measure-well-defined}, we can assume they are of the same rank without loss of generality. Then, 
\begin{subequations}
\begin{align}
    C(H_1)\cup C(H_2)
    &= \bigcup_{\vx\in H_1} \{\vx\bmomega:\bmomega\in \alphabeteos^\infty\}
    \cup \bigcup_{\vx\in H_2} \{\vx\bmomega:\bmomega\in \alphabeteos^\infty\} \\
    &= \bigcup_{\vx\in H_1 \cup H_2} \{\vx\bmomega:\bmomega\in \alphabeteos^\infty\} \qquad\justification{$H_1$ and $H_2$ equal rank and disjoint} \\
    &=C(H_1\cup H_2)
\end{align}
\end{subequations}
which leads to
\begin{subequations}
\begin{align}
    \P_0(C(H_1)\cup C(H_2))
    &= \P_0(C(H_1\cup H_2)) = \sum_{\vx\in H_1\cup H_2} \pASM(\vx)
    = \P_0(C(H_1))+\P_0(C(H_2)).
\end{align}
\end{subequations}
Hence, $\P_0$ is finitely additive.

Now, equip $\alphabeteos$ with the discrete topology. Since $\alphabeteos$ is finite, it is compact under the discrete topology and so is $\alphabeteos^\infty$ by \cref{thm:tychonoff}. Then, by properties of the product topology over discrete finite spaces, all cylinder sets in $\alphabeteos^\infty$ are compact. To apply \cref{lem:pre-measure-continuity}, let $C_1\supset C_2\supset \dotsm$ be a decreasing sequence of cylinder sets with empty intersection. Suppose to the contrary that $\P_0\left(\bigcap_n C_n\right)>0$. This would imply that all $C_n$ are nonempty (any of these being empty would result in a measure 0). However, by Cantor's intersection theorem\footnote{Cantor's intersection theorem states that a decreasing sequence of nonempty compact sets have a nonempty intersection. A version of this result in introductory real analysis is the Nested Interval Theorem.}, $\bigcap_n C_n$ is nonempty, contradicting the assumption. Hence, $\P_0\left(\bigcap_n C_n\right)=0$, and by \cref{lem:pre-measure-continuity}, $\P_0$ is countably additive.
\end{proof}

\subsection{Details in \cref{sec:extension}} 

\subsubsection{Carath\'eodory's Extension Theorem}
\label{sec:extension-proof}
\caratheodory*
\begin{proof}[Proof Sketch]
First, construct an outer measure by approximation with countable coverings. Then, show that the collection of sets that is measurable with respect to this outer measure is a $\sigma$-algebra $\calF$ that contains $\mathcal{A}$. Finally, restricting the outer measure to this $\sigma$-algebra, one is then left with a probability space. 
To show minimality, one can show that $\calF$ is contained in any $\sigma$-algebra that contains $\calA$. Uniqueness is given by applying Dynkin's $\pi$-$\lambda$ theorem (Theorem 3.2 in \citealp{billingsley1986}).

Great care must be taken in each step involved in the outline above. To address these is well beyond the scope of this paper and we refer reader to the many excellent texts with a proof of this theorem, such as Chapter 12 in \citet{royden1988real} and Chapter 11 in \citet{billingsley1986}.
\end{proof}

\subsubsection{The Space of Non-measurable Sets} \label{sec:nonmeasurable}

Non-measurable sets are, in general, difficult to find. Even when we can exhibit such sets, they tend to be very abstract and counter-intuitive. Vitali's and Bernstein's sets are two prominent examples for the Lebesgue measure. \citet{blackwell-1996} offers a construction of a non-measurable set in the cylinder $\sigma$-algebra.\footnote{The following assumes basic familiarity with the theory of ordinal numbers. Readers without such background may skip to the last paragraph for conclusion.}

As another approach to understand this better, we can consider how our collection $\sigma(\overcalC)$ of all measurable sets, i.e., our $\sigma$-algebra, is constructed from our algebra $\overcalC$ of cylinder sets (as opposed to simply knowing from Carath\'eodory's Extension Theorem that it exists). 
Concretely, as in \S1.6 in \citet{folland-1999}, we can intuitively consider the following process to build from the collection of cylinder sets $\overcalC$, which is a countable collection, all the way up to its generated $\sigma$-algebra, whose cardinality is unknown just yet:
\begin{itemize}
    \item Let $\overcalC_0=\overcalC$,
    \item Let $\overcalC_1$ be the set that includes all countable unions of sets in $\overcalC_0$ or the complements of such,
    \item Repeat this process to build $\overcalC_n$ for every $n\in\N$.
\end{itemize}
One might then take the union $\bigcup_{n\in\N}\overcalC_n$ of this increasing sequence of collections of sets, and ask if it is the same as $\sigma(\overcalC)$. In general, the answer is no (as one might expect if one is familiar with the Borel Hierarchy). However, we can obtain $\sigma(\overcalC)$ if we perform this construction for every countable ordinal $\alpha$. Abbreviating the operation in the second step above as $\delta$, i.e., $\overcalC_1=\delta(\overcalC_0)$, and letting $\omega_1$ be the collection of all countable ordinals,\footnote{$\omega_1$ is, in fact, the same as the first uncountable ordinal. Its existence (and hence the existence of the collection of all countable ordinals) can be guaranteed by exhibiting a well-ordered uncountable set using the Axiom of Choice.} we can define
\begin{align}
    \overcalC_\alpha = \begin{cases}
        \delta(\overcalC_\beta) & \text{if } \alpha=\beta+1 \text{ for some } \beta\in\omega_1, \\
        \bigcup_{\beta\in\omega_1:\beta<\alpha} \overcalC_\beta & \text{otherwise}.
    \end{cases}
\end{align}
This will give us the desired set as follows:
\begin{proposition}[Proposition 1.23, \citealp{folland-1999}] \label{prop:folland-123}
$\sigma(\overcalC)=\bigcup_{\alpha\in\omega_1} \overcalC_\alpha$.
\end{proposition}

Next, we recall the following basic fact from cardinality theory.
\begin{proposition}[Proposition 0.14, \citealp{folland-1999}] \label{prop:folland-014}
If $\mathrm{card}(A)\leq 2^{\aleph_0}$ and $\mathrm{card}(X_\alpha)\leq2^{\aleph_0}$ for all $\alpha\in A$, then $\mathrm{card}\left(\bigcup_{\alpha\in A}X_\alpha\right)\leq2^{\aleph_0}$.
\end{proposition}
Noting that $\mathrm{card}(\omega_1)\leq 2^{\aleph_0}$ and $\mathrm{card}(\overcalC)={\aleph_0}$, we can conclude that $\mathrm{card}(\sigma(\overcalC))\leq 2^{\aleph_0}$ from \cref{prop:folland-123} and \cref{prop:folland-014}. In other words, the cardinality of $\sigma(\overcalC)$ is at most that of the continuum, and since $\mathrm{card}\big(\mathcal{P}(\alphabeteos^\infty)\big)=2^{2^{\aleph_0}}(=\beth_2)$, $\sigma(\overcalC)$ is, in terms of cardinality, an almost negligible subset of $\mathcal{P}(\alphabeteos^\infty)$!  That is, most subsets in $\alphabeteos^\infty$ are non-measurable---though explicit examples have rarely been constructed \cite{blackwell-1996}.  \Cref{sec:rv-proof} below establishes that common subsets of $\alphabeteos^\infty$ that we work with are measurable.

\subsection{Proofs in \cref{sec:rv}} \label{sec:rv-proof}

\rvMeasurable*
\begin{proof}
  To show that $X$ is measurable, it suffices to show the measurability of preimages of a generating set\footnote{A set $G$ is said to be a generating set of a $\sigma$-algebra $\calF$ if $\calF$ is the smallest $\sigma$-algebra that contains $G$.} of the $\sigma$-algebra $\sigma(\calC)$ on $\alphabet^* \cup \alphabet^\infty$. Such a generating set is formed by the thin cylinders $C(\bm{x}) \defeq C(\{\bm{x}\})$ for $\bm{x} \in \alphabet^*$. (Recall that cylinders in $\alphabet^* \cup \alphabet^\infty$ are defined by \cref{eq:rv-cylinder}.)  Given $\bm{x}\in\alphabet^*$:
\begin{subequations}
\begin{align}
X^{-1}(C(\bm{x}))
=& X^{-1}(\{\bm{x}\bmomega: \bmomega\in\alphabet^*\cup \alphabet^\infty\}) \\
=& X^{-1}(\{\bm{x} \bmomega: \bmomega\in\alphabet^* \})  \cup X^{-1}(\{\bm{x} \bmomega: \bmomega\in\alphabet^\infty \}) \\
=& \left(\bigcup_{\bmomega\in \alphabet^*} \overline{C}(\bm{x}\bmomega\,\eos)\right)  \cup \left(\overline{C}(\bm{x})\cap \bigcap_{k=1}^\infty A^\mathsf{c}_k \right)
\end{align}
\end{subequations}
Note that the set $A_k$ above, defined by \cref{eq:term-at-k}, is a cylinder of $\alphabeteos^\infty$, representing the event of terminating by step $k$.
Then, from the derivation above, we can see that $X^{-1}(C(\vx))$ is formed by countable operations over measurable sets (cylinders) of $\alphabeteos^\infty$, and is hence measurable.
So $X$ is a measurable function.
\end{proof}

\begin{proposition} \label{prop:rv-singleton-measurable}
In measure space $(\alphabet^*\cup\alphabet^\infty, \sigma(\calC))$, $\{\vx\}$ is measurable for all $\vx\in\alphabet^*$.
\end{proposition}
\begin{proof}
We will show that $\{\vx\}=C(\vx)\setminus \bigcup_{a\in\alphabet} C(\vx a)$ and hence is measurable.
By definition in \cref{eq:rv-cylinder}, for any $\vx\in\alphabet^*$,
\begin{subequations}
\begin{align}
    C(\vx) &=\{\vx\bmomega: \bmomega\in \alphabet^*\cup\alphabet^\infty \} \\
    &=\{\vx\bmomega: \bmomega\in \alphabet^* \} \cup \{\vx\bmomega: \bmomega\in \alphabet^\infty \}
\end{align}
\end{subequations}
where
\begin{subequations}
\begin{align}
    \{\vx\bmomega: \bmomega\in \alphabet^* \} = \{\vx\} \cup \bigcup_{a\in\alphabet} \{\vx a\bmomega: \bmomega\in \alphabet^* \}
\end{align}
\end{subequations}
and
\begin{align}
    \{\vx\bmomega: \bmomega\in \alphabet^\infty \} = 
    \bigcup_{a\in\alphabet} \{\vx a\bmomega: \bmomega\in \alphabet^\infty \}.
\end{align}
So
\begin{subequations}
\begin{align}
    C(\vx) &= \{\vx\} \cup \bigcup_{a\in\alphabet} \bigg(
        \{\vx a\bmomega: \bmomega\in \alphabet^* \}
     \cup\{\vx a\bmomega: \bmomega\in \alphabet^\infty \}\bigg) \\
    &= \{\vx\} \cup \bigcup_{a\in\alphabet} C(\vx a)
\end{align}
\end{subequations}
where the union is disjoint.  This implies $\{\vx\}=C(\vx)\setminus \bigcup_{a\in\alphabet} C(\vx a)$ as desired. 
\end{proof}

\begin{proposition} \label{prop:rv-infinite-measurable}
In the measure space $(\alphabet^*\cup\alphabet^\infty, \sigma(\calC))$, $\alphabet^\infty$ is measurable.
\end{proposition}
\begin{proof}
First, $\alphabet^*\cup\alphabet^\infty$ is the entire outcome space, which is measurable by the definition of $\sigma$-algebra. Notice that
\begin{align}
    \alphabet^\infty = (\alphabet^*\cup\alphabet^\infty) \setminus 
    \bigcup_{\vx\in\alphabet^*} \{\vx\}.
\end{align}
Since each $\{\vx\}$ in the above is measurable by \cref{prop:rv-singleton-measurable} and $\alphabet^*$ is a countable set, $\alphabet^\infty$ is then measurable.
\end{proof}

The measurability of $\alphabet^\infty$ in
$(\alphabet^*\cup\alphabet^\infty, \sigma(\calC))$ (\cref{prop:rv-infinite-measurable}) was assumed by our definition of tightness (\cref{def:tight}).  As we have also established that each $\{\vx\}$ is measurable (\cref{prop:rv-singleton-measurable}), we can give an alternative characterization.

\begin{proposition} \label{prop:seq-tight}
A sequence model $(\alphabet^*\cup\alphabet^\infty, \sigma(\calC), P)$ is tight if and only if $\sum_{\vx\in\alphabet^*} P(\{\vx\})=1$.
\end{proposition}
\begin{proof}
We defined a sequence model to be tight if and only if $P(\alphabet^\infty)=0$ (\cref{def:tight}).
By \cref{prop:rv-singleton-measurable,prop:rv-infinite-measurable}, we can write
\begin{subequations}
\begin{align}
    1 = P(\alphabet^*\cup\alphabet^\infty)
    &=P(\alphabet^\infty)+P(\alphabet^*) & \justification{finite additivity} \\
    &=P(\alphabet^\infty)+\sum_{\vx\in\alphabet^*} P(\{\vx\}). & \justification{countable additivity}
\end{align}
\end{subequations}
Hence, a sequence model is tight if and only if $\sum_{\vx\in\alphabet^*} P(\{\vx\})=1$.
\end{proof}

\section{Proofs on Characterizing Tightness (\cref{sec:tightness})}\label{sec:termination-proof}

\subsection{Proof of \cref{lem:durrett-435}}

The result below is stated without proof as Exercise 4.3.5 in \citet{durrett_2019}.

\durrettConditional*
\begin{proof}
First, recall an elementary inequality that for $x>0$,
\begin{align}
    x-1\geq \log x \quad 
    \Leftrightarrow \quad 1-x \leq \log\frac{1}{x}. \label{ineq:durrett-435}
\end{align}
Note that $\P(\cap_{m=1}^nA_m^\mathsf{c})>0$ for any $n$, for otherwise the conditional probabilities would be undefined.
Let $p_n \defeq \P(\cap_{m=1}^nA_m^\mathsf{c})$. 
Then we have that $p_n>0$ for all $n$, and
\begin{subequations}
\begin{align}
\infty
&=\sum_{n=2}^\infty \P(A_n\mid\cap_{m=1}^{n-1}A_m^\mathsf{c}) \\
&=\sum_{n=2}^\infty 1-\P(A_n^\mathsf{c}\mid\cap_{m=1}^{n-1}A_m^\mathsf{c}) \\
&=\lim_{N\to\infty}\sum_{n=2}^N 1-\P(A_n^\mathsf{c}\mid\cap_{m=1}^{n-1}A_m^\mathsf{c}) \\
&\leq\lim_{N\to\infty}\sum_{n=2}^N \log 1/\P(A_n^\mathsf{c}\mid\cap_{m=1}^{n-1}A_m^\mathsf{c}) &\justification{by \cref{ineq:durrett-435}} \\
&=\lim_{N\to\infty}\sum_{n=2}^N \log \frac{\P(\cap_{m=1}^{n-1}A_m^\mathsf{c})}{\P(\cap_{m=1}^{n}A_m^\mathsf{c})} \\
&=\lim_{N\to\infty}\sum_{n=2}^N \log \frac{p_{n-1}}{p_n} \\
&=\lim_{N\to\infty}\sum_{n=2}^N (\log p_{n-1}- \log p_n) \\
&=\lim_{N\to\infty} (\log p_1 - \log p_N) \\
&=\log p_1 - \lim_{N\to\infty} \log p_N
\end{align}
\end{subequations}
which implies that
\begin{subequations}
\begin{align}
    &\lim_{N\to\infty} \log p_N = -\infty \\
    \Leftrightarrow\quad & \lim_{N\to\infty} p_N = 0 \\
    \Leftrightarrow\quad & \lim_{N\to\infty} \P(\cap_{m=1}^N A_m^\mathsf{c}) = 0 \\
    \Leftrightarrow\quad & \P(\cap_{m=1}^\infty A_m^\mathsf{c}) = 0. &\justification{by continuity of measure}
\end{align}
\end{subequations}
\end{proof}

\subsection{Proof of \cref{prop:div-implies-tight}} 
\label{sec:div-implies-tight-proof}

\divergentLowerBound*
\begin{proof}
In the proof, we rename the index $t$ to $n$ to match the usual presentation of the Borel-Cantelli lemmata.  We are given that $\pASM(\textsc{eos}\mid\bm{x})\geq f(n)$ for all $\bm{x}\in\alphabet^{n-1}$. To apply \cref{lem:durrett-435}, we observe that
\begin{subequations}
\begin{align}
    A_n \cap (A_1^\mathsf{c} \cap \dotsm \cap A_{n-1}^\mathsf{c}) =& \{\bmomega \in \alphabeteos^\infty: \omega_n=\textsc{eos}\}~\cap \left(
    \bigcap_{i=1}^{n-1} \{\bmomega \in \alphabeteos^\infty: \omega_i\not=\textsc{eos}\}
    \right) \\
    =& \{\bmomega \in \alphabeteos^\infty: \bmomega=\textsc{eos}, \forall~i<n, \bmomega\not=\textsc{eos} \} \\
    =& \{\bmomega \in \alphabeteos^\infty: \text{$\bmomega$'s first \textsc{eos} is at position $n$}\}
\end{align}
\end{subequations}
and similarly
\begin{equation}
    A_1^\mathsf{c} \cap \dotsm \cap A_{n-1}^\mathsf{c} = \{ \bmomega \in \alphabeteos^\infty: \text{there is no \textsc{eos} in $\bmomega$'s first $n-1$ positions}\}
\end{equation}
Setting $G\defeq\{\bmomega\,\textsc{eos}:\bmomega\in \alphabet^{n-1} \}\subset\alphabeteos^n$, we get
\begin{subequations}
\begin{align}
    \P(A_n \mid A_1^\mathsf{c} \cap \dotsm \cap A_{n-1}^\mathsf{c}) \label{eq:weighted-eos} &= \frac{
          \P(A_n \cap (A_1^\mathsf{c} \cap \dotsm \cap A_{n-1}^\mathsf{c}))
        }{
          \P(A_1^\mathsf{c} \cap \dotsm \cap A_{n-1}^\mathsf{c})
        } \\
    &= \frac{\P(\overC(G))
        }{
        \P(\overC(\alphabet^{n-1}))} &\justification{definition of $G$} \\
    &=\frac{
        \sum_{\bmomega\in \alphabet^{n-1}} \pASM(\textsc{eos}\mid\bmomega)\pASM(\bmomega)}
        {\sum_{\bmomega\in \alphabet^{n-1}} \pASM(\bmomega)} &\justification{by \cref{eq:pre-measure}} \\
    &\geq \frac{\sum_{\bmomega\in \alphabet^{n-1}} f(n)\pASM(\bmomega)
    }{
    \sum_{\bmomega\in \alphabet^{n-1}} \pASM(\bmomega)
    } &\justification{definition of $f(n)$} \\
    &= f(n) \frac{\sum_{\bmomega\in \alphabet^{n-1}} \pASM(\bmomega)
    }{
    \sum_{\bmomega\in \alphabet^{n-1}} \pASM(\bmomega)
    } \\
    &= f(n).
\end{align}
\end{subequations}
Since $\sum_{n=1}^\infty f(n)=\infty$ and hence $\sum_{n=2}^\infty f(n)=\infty$,
the above inequality shows that the condition of \cref{lem:durrett-435} holds.
Hence by \cref{lem:durrett-435}, the event of a string never terminating, i.e., $\cap_{k=1}^\infty A_k^\mathsf{c}$, has probability measure $\P(\cap_{k=1}^\infty A_k^\mathsf{c})=0$.

In summary, if the \eos probability of a language model is lower-bounded at ever steps by the terms of a divergent series, then the event that this language model terminates has probability 1.
\end{proof}

\subsection{Proof of \cref{cor:durrett-434}}
\label{sec:durrett-434-proof}

To show \cref{cor:durrett-434}, we first show the following simple consequence of Borel--Cantelli.

\begin{corollary} \label{cor:io-implies-div}
If $\P(A_n \text{ i.o.})=1$, then $\sum_{n=1}^\infty \P(A_n) = \infty$.
\end{corollary}
\begin{proof}
Suppose to the contrary that $\sum_{n=1}^\infty \P(A_n) < \infty$, then, by Borel--Cantelli I (\cref{thm:bc1}), $\P(A_n \text{ i.o.})=0$, which contradicts the assumption. Hence, $\sum_{n=1}^\infty \P(A_n) = \infty$.

\end{proof}

\cref{cor:durrett-434} below is also stated without proof as Exercise 4.3.4 in \citet{durrett_2019}.
\durrettSequence*
\begin{proof}
We can use a product measure to construct a sequence of independent events $\{A_n\}_{n=1}^{\infty}$ such that $\P(A_n)=p_n$.
(The product measure ensures independence.) 
Then, by definition in \cref{eq:defn-io},
\begin{align}
\{A_n \text{ i.o.}\}^\textsf{c} = \bigcup_{m=1}^\infty \bigcap_{n\geq m} A_n^\textsf{c}
\end{align}
So,
\begin{subequations}
\begin{align}
1-\P(A_n \text{ i.o.})&=\P\left(\bigcup_m \bigcap_{n\geq m} A_n^\textsf{c}\right) \\
&=\lim_{m\to\infty} \P\left(\bigcap_{n\geq m} A_n^\textsf{c}\right) \\
&=\lim_{m\to\infty} \prod_{n\geq m} \P(A_n^\textsf{c}) &\justification{$A_n$ are independent by construction} \\
&=\lim_{m\to\infty} \prod_{n\geq m} (1-p_n)
\end{align}
\end{subequations}

\paragraph{$(\Rightarrow)$:} Assume $\prod_{n=1}^\infty (1-p_n) = 0$. Then, for any $m$,
\begin{equation}
0= \prod_{n\geq 1} (1-p_n) =\underbrace{\left(\prod_{1\leq n<m} (1-p_n)\right)}_{>0}
\left(\prod_{n\geq m} (1-p_n)\right)
\end{equation}
So it must the case that, for any $m$, $\prod_{n\geq m} (1-p_n)=0$. Therefore,
\begin{equation}
1-\P(A_n \text{ i.o.}) = \lim_{m\rightarrow \infty} \prod_{n\geq m} (1-p_n) = 0
\end{equation}
which implies $\P(A_n \text{ i.o.})=1$.
\Cref{cor:io-implies-div} implies that $\sum_{n=1}^{\infty} p_n=\infty$.

\paragraph{$(\Leftarrow)$:} Assume $\sum_{n=1}^\infty p_n=\infty$. 
Then by Borel--Cantelli II (\cref{thm:bc2}), $\P(A_n \text{ i.o.})=1$ which implies
\begin{equation}
    0 = 1-\P(A_n \text{ i.o.}) = \lim_{m\rightarrow \infty} \prod_{n\geq m} (1-p_n)
\end{equation}
Observe that $\Big\{\prod_{n\geq m} (1-p_n)\Big\}_m$ is a non-decreasing sequence in $m$; to see this, note that as $m$ grows larger we multiply strictly fewer values $(1 - p_n) \in (0, 1]$. 
However, since we know the sequence is non-negative and tends to $0$, it follows that for \emph{any} $m$, we have
\begin{align}
\prod_{n\geq m} (1-p_n) =0.
\end{align}
It follows that, for any $m$, we have
\begin{equation}
    \prod_{n=1}^\infty (1-p_n) = \prod_{n<m} (1-p_n) \underbrace{\prod_{n\geq m} (1-p_n)}_{=0} = \prod_{n<m} (1-p_n) \cdot 0 = 0.
\end{equation}
\end{proof}

\subsection{Proof of \cref{prop:lm-tight-main}} 
\label{sec:lm-tight-main-proof}

\lmTightMain*
\begin{proof}
Recall the definition of $\ptildeEOS$, as previously defined in \cref{eq:ptildeeos-set}, is 
\begin{align}
\ptildeEOS(t)\defeq\P(A_t\mid A_1^\mathsf{c}\cap \dotsm \cap A_{t-1}^\mathsf{c}).
\end{align}

\paragraph{Case 1.} Suppose that $\ptildeEOS(t)<1$ for all $t$. Consider the termination probability again:
\begin{subequations}
\begin{align}
\P\left(\bigcap_{t=1}^\infty A_t^\mathsf{c}\right) 
&= \lim_{T\to\infty} \P\left(\bigcap_{t=1}^T A_t^\mathsf{c}\right) \\
&= \lim_{T\to\infty} \prod_{t=1}^T \P(A_t^\comp \mid A_1^\comp \cap \dotsm \cap A_{t-1}^\comp) \\
&= \lim_{T\to\infty} \prod_{t=1}^T (1-\widetilde{p}_\eos(t)) \\
&= \prod_{t=1}^\infty (1-\widetilde{p}_\eos(t)). \label{eq:term-prob-inf-prod}
\end{align}
\end{subequations}
In the above, we have assumed that $\P(A_1^\comp \cap \dotsm \cap A_{t}^\comp)>0$ for all $t$, which is true by assumption that $\ptildeEOS(t)<1.$. 
Hence, by \cref{cor:durrett-434}, \cref{eq:term-prob-inf-prod} is $0$ if and only if $\sum_t \widetilde{p}_\eos(t)=\infty$.

\paragraph{Case 2.} If $\widetilde{p}_\eos(t)=1$ is true for some $t=t_0$, then $\P(A_1^\comp \cap \dotsm \cap A_{t_0}^\comp)=0$ and hence $\P\left(\bigcap_{t=1}^\infty A_t^\comp\right)=0$ and such a language model is guaranteed to terminate at $t_0$.
\end{proof}

\section{Proofs for Analyses of Common Language Models (\cref{sec:lm-analysis})}

\subsection{Proofs for FSSMs (\cref{sec:sfssm})}\label{app:n-gram}

\subsubsection{Proofs for Stochastic FSSMs}\label{sec:sfslm-proof}
\sfslmTight*
\begin{proof} We refer to a state $q$ as \defn{initial} if $s_q > 0$ and as \defn{final} if $t_q > 0$.  (These are sometimes called source and sink states.)
We prove each direction of the theorem in turn:
  \paragraph{($\Rightarrow$):}  Assume the SFSSM is tight. Let $q$ be an accessible state.  Since the SFSSM has at least one positive-probability path from an initial state, there is a positive probability of reaching $q$ during generation.  If there were no positive-probability path from $q$ to a final state, then the SFSSM would never terminate on the occasions when it reached $q$, contradicting the assumption of tightness. 
  Hence $q$ must be co-accessible. 

\paragraph{($\Leftarrow$):} 
Assume that all accessible states are co-accessible. We construct a Markov chain whose states are the SFSSM's accessible states $Q_A\subseteq \{1, \ldots, Q\}$ together with an \eos state.  In this Markov chain, the initial probability of $q$ is given by $s_q$ when $q \in Q_A$ and by $0$ when $q=\eos$; the transition probability from $q$ to $q'$ is given by $\mP_{qq'}$ when $q,q' \in Q_A$, by $t_q$ when $q \in Q_A$ and $q'=\eos$, by $1$ when $q=q'=\eos$, and by $0$ otherwise. 
The probability that the Markov chain is in state $q \in Q_A$ after $t$ steps equals the probability that the SFSSM is in state $q$ after $t$ steps (note that the SFSSM never reaches any state $q \notin Q_A$).  The probability that it is in state $\eos$ after $t$ steps equals the probability that the SFSSM has terminated after $\leq t$ steps.  

Clearly $\eos$ is an \defn{absorbing state} of the Markov chain, meaning that  once the Markov chain reaches this state, it never leaves. A fundamental result on finite-state Markov chains \cite[Theorem 11.3]{grinstead1997} is that if every state can reach an absorbing state, then with probability 1, the chain reaches an absorbing state (``is absorbed'') in finite time.   Every state can in fact reach $\eos$, by coaccessibility of $Q_A$. This further implies that \eos is the \emph{only} absorbing state (as an absorbing state cannot reach any other state).  So by the result cited above, the Markov chain reaches \eos with probability 1 in finite time. Consequently, the SFSSM terminates after finitely many steps with probability 1; that is, the SFSSM is tight.  
\end{proof}

\ngramMLETight*
\begin{proof}
The SFSSM for an $n$-gram model has states that correspond to $(n-1)$-grams and transitions that correspond to characters (unigrams), as illustrated by \cref{fig:sfssm}. 
When the SFSSM's probabilities are estimated with MLE, the accessible states are $(n-1)$-grams that have appeared in some string in the corpus. 
Such states must also be co-accessible so that they can generate the rest of that string.
Hence, by \cref{thm:sfslm-tight}, this SFSSM is tight.
\end{proof}

\subsubsection{Proofs for Substochastic FSSMs}\label{sec:subfslm-proof}

To prove \cref{thm:sub-fslm-tight}, we will make use of the following useful lemma.

\begin{lemma}\label{lem:sub-fslm-eigen}
Let $\mP'$ be the transition sum matrix of a trimmed substochastic FSSM.  Then $\rho(\mP')<1$ where $\rho(\cdot)$ denotes the spectral radius.
\end{lemma}
\begin{proof}
To begin with, we wish to apply the following result, which connects the row sums of a matrix to its spectral radius. Below, $M_n$ denotes the set of $n\times n$ matrices, and $\vertiii{A}_\infty=\max_{1\leq i \leq n} \sum_{j=1}^n |A_{ij}|$ denotes the operator $\infty$-norm.
\begin{proposition}[\S6.2.P8, \citealp{horn2013}] \label{prop:hj-628}
    For any $A\in M_n$, $\rho(A)\leq \vertiii{A}_\infty$. Additionally, if $A$ is irreducible and not all absolute row sums of $A$ are equal, then $\rho(A)<\vertiii{A}_\infty$.
\end{proposition}
However, the transition sum matrix $\mP$ of a substochastic FSSM may be reducible, whereas the irreducibility condition in \cref{prop:hj-628} cannot be dropped. Hence, we need to ``decompose'' $\mP'$ in a way that recovers irreducibility. We use the \emph{Frobenius normal form} (also known as \emph{irreducible normal form}) to achieve this.
\begin{proposition}[\S8.3.P8, \citealp{horn2013}] \label{prop:hj838}
        Let $A\in M_n$ be non-negative. Then, either $A$ is irreducible or there exists a permutation matrix $\Pi$ such that
        \begin{align}
        \Pi^\top A \Pi =
        \begin{bmatrix}
            A_1 & & \ast \\
                & \ddots & \\
            \mathbf{0} & & A_k
        \end{bmatrix}
        \end{align}
        is block upper triangular, and each diagonal block is irreducible (possibly a $1\times 1$ zero matrix). This is called a \defn{Frobenius normal form} (or \defn{irreducible normal form}) of $A$. Additionally, $\lambda(A)=\lambda(A_1)\cup\dotsm\cup\lambda(A_k)$ where $\lambda(\cdot)$ denotes the set of eigenvalues of a matrix.
\end{proposition}

Notice that the permutation in the Frobenius normal form merely renumbers the states of the trimmed FSSM.  We may check that as a result, the termination probability given in \cref{thm:sub-fslm-tight} is unchanged:\footnote{The equalities here use the fact that the inverse of a permutation matrix $\Pi$ is its transpose: $\Pi\,\Pi^\top = I$.}
\begin{align}
    (\Pi^\top\vsource')^\top (\Pi^\top \mP' \Pi)^k (\Pi^\top\vtarget') 
    = (\vsource'^\top \Pi) (\Pi^\top \mP'^k \Pi)(\Pi^\top \vtarget') 
    = \vsource'^\top \mP'^k \vtarget'
\end{align}
Hence, with an appropriate renumbering, we may assume without loss of generality that $\mP$ is already given in Frobenius normal form
\begin{align}
    \mP' = \begin{bmatrix}
        \mP'_1 &  & \ast \\
        & \ddots & \\
        \mathbf{0} & & \mP'_k
    \end{bmatrix}
\end{align}
where each $\mP'_i$ is irreducible.

Since the transition sum matrix $\mP'$ of a trimmed substochastic FSSM is a substochastic matrix, each $\mP'_i$ is also substochastic. In fact, each $\mP'_i$ is \emph{strictly} substochastic, meaning that there is at least one row that sums to less than 1. To see this, suppose to the contrary that there is a stochastic $\mP'_i$. Since the FSSM is trimmed, every state is both accessible and co-accessible. Being accessible implies that there is a positive probability of reaching every state in $\mP'_i$. However, the stochasticity of $\mP'_i$ forces the corresponding $\vtarget'$ entries to be 0. Hence, none of these states can transition to \eos, meaning that they're not co-accessible, contradicting the assumption. Hence, every $\mP'_i$ is strictly substochastic. Then, either all row sums of $\mP'_i$ are less than 1 (in which case $\vertiii{\mP'_i}_\infty < 1$) or some row sums are 1 and some are less than 1 (in which case $\vertiii{\mP'_i}_\infty = 1$ and $\mP'$ has unequal absolute row sums).  In either case, \cref{prop:hj-628} implies that $\rho(\mP'_i)<1$, for all $1\leq i\leq k$.

Finally, the last sentence of \cref{prop:hj838} entails that $\rho(\mP')=\max\{\rho(\mP'_1),\dots,\rho(\mP'_k)\}$.  Since each $\rho(\mP'_i)<1$, we have $\rho(\mP')<1$.
\end{proof}

\trimSfslmTight*
\begin{proof}
By \cref{lem:sub-fslm-eigen}, $\rho(\mP')<1$, in which case $\mI-\mP'$ is invertible and the Neumann series $\sum_{k = 0}^\infty \mP'^k = \mI+\mP'+\mP'^2+\dotsm$ converges to $(\mI-\mP')^{-1}$ \citep[\S5.6]{horn2013}. Thus

\begin{align}
P(\alphabet^\ast) &= \sum_{k = 0}^\infty P(\alphabet^k) = \sum_{k = 0}^\infty \vsource'^\top \mP'^k \vtarget' = \vsource'^\top \left(\sum_{k=0}^\infty \mP'^k\right) \vtarget' = \vsource'^\top (\mI - \mP')^{-1} \vtarget'.
\end{align}
\end{proof}

\subsection{Proofs for Transformer Result (\cref{sec:transformer-thms})}
\label{sec:transformer-proofs}

Again, the following theorem is well-known:
\compact*
\begin{proof}
Let $\{U_\alpha\}_{\alpha\in\mathcal{A}}$ be any open cover of $f(X)$. By continuity, $f^{-1}(U_\alpha)\subset X$ is open for any $\alpha\in\mathcal{A}$, and hence $\{f^{-1}(U_\alpha)\}_{\alpha\in\mathcal{A}}$ is also an open cover of $X$. By the compactness of $X$, there is a finite sub-cover $\{f^{-1}(U_{\alpha_i})\}_{i=1}^n$, in which case $\{U_{\alpha_i}\}_{i=1}^n$ forms a finite sub-cover for $f(X)$.
\end{proof}

\transformerCompact*
\paragraph{Note.} 
We make use of the following notations in the proof below: $\triangle_{t-1}=\{\bm{y}\in\R^t:\bm{y}\geq 0, \bm{1}^\top\bm{y}=1\}$ denotes the $(t-1)$-dimensional simplex; $B_r(\bm{z})=\{\bm{v}\in\R^n:\text{dist}(\bm{z},\bm{v})<r\}$ denotes the open ball centered at $\bm{z}$ with radius $r$; $\overline{A}$ denotes the closure of set $A$.
\begin{proof} Let $K_0=K$.  In an autoregressive transformer, each of the $n$ layers consists of two blocks: a self-attention block and a feedforward block. We will use induction on the $2n$ blocks to build up compact sets $K_1, K_2, \ldots, K_{2n}$ that contain the output vectors of these respective blocks, and then take $K' = K_{2n}$.

The self-attention block is a function on $(\R^d)^+\to(\R^d)^+$. So, let $t\in\Z_{>0}$ be arbitrary and consider any sequence of input vectors $(\bm{v}_1,\dots,\bm{v}_t)$ such that for all $i$, $\bm{v}_i\in K_0$. 
Denote the output vectors of the attention block with $(\bm{v}_1', \dots,\bm{v}_t')$. 
By definition of attention, each output vector $\bm{v}_j'=\sum_{i=1}^t \alpha^{(j)}_i\bm{v}_i$ where $\bm{\alpha}^{(j)}\in\triangle_{t-1}$ are the attention weight vectors obtained through the softmax function. 
Compact sets in $\R^d$ are bounded (by the Heine--Borel theorem), and hence there exists $M>0$ such that $K_0\subseteq\overline{B_M(0)}$. Noting that the norm function $\Vert\cdot\Vert$ on $\R^d$ is convex, we have the following
\begin{subequations}
\begin{align}
\Vert\bm{v}_j'\Vert&=\left\Vert\sum_{i=1}^t\alpha^{(j)}_i\bm{v}_i\right\Vert \\
&\leq\sum_{i=1}^t \alpha^{(j)}_i \Vert\bm{v}_j\Vert \label{eq:used-jensen} \tag{$\ast$} \\
&\leq\sum_{i=1}^t \alpha^{(j)}_i M=M \label{eq:attn-output-bound}
\end{align}
\end{subequations}
where (\ref{eq:used-jensen}) results from Jensen’s inequality. \cref{eq:attn-output-bound} shows that each of the output vectors $\bm{v}_j'$ lies in $\overline{B_M(0)}$ which is compact. Hence, setting $K_1=\overline{B_M(0)}$, we have shown that, for any $t\in\Z_{>0}$, the attention block maps $K_0^t$ into $K_1^t$. 

Note that we \emph{cannot} use \cref{thm:compact} here because the attention block defines a different function on $\R^{t\times d}\to\R^{t\times d}$ for each $t$, and \cref{thm:compact} only implies that there exists a separate \emph{length-dependent} output compact set $K_t\subset\R^{t\times d}$ for each $t$, which is different from this lemma's statement. 

The feedforward function is a continuous function on $\R^d\to\R^d$, and therefore, by \cref{thm:compact}, maps its input compact set $K_1$ to an output compact set, which we call $K_2$.

Finally, residual connections and layer norms are also continuous functions acting on each of the input vectors, and hence by the same reasoning would also preserve compactness.

Now we can use induction and show that there exist compact sets $K_3, K_4, \dots, K_{2n-1}, K_{2n}$ where $K_{2n}$ contains the output set of the final layer. Set $K'=K_{2n}$ and we have proven the statement.
\end{proof}

\transformerMain*
\begin{proof}
Given the Transformer, there exists a fixed compact set $K$ that will contain all inputs $\bm{v}_i \in \R^d$ to the first layer.  This is true because each $\bm{v}_i$ is the sum of a word embedding, which falls in a finite set since $\alphabeteos$ is finite, and a position embedding, which lies in the compact set $[-1,1]^d$. Hence, by \cref{lem:transformer-compact}, there exists a fixed compact set $K'$ that contains all output embedding vectors (regardless of how long the sequence is). 

The final output probability is given by a multiplication with the word embedding matrix followed by the softmax function as in \cref{eq:transformer-softmax}. This process amounts to composing two continuous functions. In particular, we can extract the \eos probability as a continuous $\R$-valued function $g_\textsc{eos}: K'\to(0,1)$ (neither 0 or 1 is in the range of the softmax function). By continuity of $g_\textsc{eos}$ and \cref{thm:compact}, $K''\defeq g_\textsc{eos}(K')\subseteq(0,1)$ is compact. Since $K''$ is compact, and hence closed, $\inf K'' \in K''$. 
Thus $\inf K'' \in (0,1)$ and in particular
$\inf K'' > 0$. Therefore, taking $\epsilon=\inf K''$, we have shown that the \eos probability of a Transformer is bounded below by some $\epsilon>0$ (regardless of the length of the sequence).

Hence, by \cref{prop:div-implies-tight}, any Transformer ASM is tight and thus defines a language model.
\end{proof}

\subsection{Proofs for RNN Result (\cref{sec:rnn-thms})}
\label{app:rnn-tight}

\rnnTight*

\begin{proof}
Let $X_t(\bmomega)$ be the random variable that is equal to the $t$\textsuperscript{th} token in an outcome $\bmomega \in \Omega$.
Also let $\vh_{\vx}$ be the hidden representation of the RNN after processing some finite list of tokens $\vx \in \alphabet^\ast$.
Further, let $\vu_x \in \R^d$ be the output embedding of $x \in \alphabeteos$,
Then for any $t \in \N$ and any $\vx \in \alphabet^t$, we have:
\begin{subequations}
\begin{align}
\P(X_{t+1} = \eos \mid \boldX_{\leq t} = \vx) 
&= \frac{\exp \ueos^{\top} \vh_{\vx}}{\sum_{y \in \alphabeteos} \exp \vu_y^{\top} \vh_{\vx}} \\
&= \frac{1}{{\sum_{y \in \alphabeteos} \exp \vu_y^{\top} \vh_{\vx}}\,/\,{\exp \ueos^{\top} \vh_{\vx}}} \\
&= \frac{1}{1 + \sum_{y \in \alphabet} \exp (\vu_y - \ueos)^{\top} \vh_{\vx}} \\
&\geq \frac{1}{1 + \sum_{y \in \alphabet} \exp \left(\Vert\vu_y - \ueos\Vert_2 \Vert\vh_{\vx}\Vert_2\right)} & \text{\normalfont (Cauchy--Schwarz)}  \\
&\geq \frac{1}{1 + \sum_{y \in \alphabet} \exp(k\Vert\vh_{\vx}\Vert_2)} \\ 
&= \frac{1}{1 + |\alphabet| \exp(k\,\Vert\vh_{\vx}\Vert_2)}
\end{align}
\end{subequations}
Now define $\Vert\widehat\vh_t\Vert_2 \defeq \sup_{\vx \in \alphabet^t} \Vert\vh_{\vx}\Vert_2$.
We then have that $\forall t \in \N$ and $\forall \vx \in \alphabet^t$:
\begin{align}
    \P(X_{t+1} = \eos \mid \boldX_{\leq t} = \vx) \ge \frac{1}{1 + |\alphabet| \exp(k \Vert\widehat\vh_t\Vert_2)}
\end{align}

Now, by \cref{prop:div-implies-tight}, we have that if $\sum_{t = 0}^\infty \frac{1}{1 + |\alphabet| \exp(k\,\Vert\widehat\vh_t\Vert_2)}$ diverges, then the language model is tight.
We will show that this condition holds if $\exists N \in \N$ such that $\forall t \geq N$, $k\Vert\widehat\vh_t\Vert_2 \le \log t$.

First, note that $\lim_{t \to \infty} \frac{1}{t}\frac{1 + |\alphabet| t}{1} = \lim_{t \to \infty} \frac{1}{t} + |\alphabet| = |\alphabet| \in (0, \infty)$.
Hence, by the limit comparison test, since $\sum_{t=1}^\infty \frac{1}{t}$ diverges, this means $\sum_{t=1}^\infty \frac{1}{1 + |\alphabet| t}$ must also diverge.

Now, suppose there exists $N$ such that that $k \, \Vert\widehat\vh_t\Vert_2 \le \log t$ for all $t \ge N$.
This implies that for $t \ge N$ we have $\frac{1}{1 + |\alphabet| \exp( k \Vert\widehat\vh_t\Vert_2)} \ge \frac{1}{1 + |\alphabet| t}$, which combined with the above and the comparison test, implies that 
$\sum_{t=N}^\infty \frac{1}{1 + |\alphabet| \exp( k \Vert\widehat\vh_t\Vert_2)}$ diverges.
This in turn means that $\sum_{t=0}^\infty \frac{1}{1 + |\alphabet| \exp( k \Vert\widehat\vh_t\Vert_2)}$ diverges.

Hence, if $k \, \Vert\widehat\vh_t\Vert_2 \le \log t$ for all sufficiently large $t$ (that is, for all $t \geq N$), then the RNN ASM is tight and thus defines a language model.

\end{proof}

\end{document}